\documentclass{article}

\usepackage{microtype}
\usepackage{graphicx}
\usepackage{subfigure}
\usepackage{booktabs} % for professional tables

\usepackage[utf8]{inputenc} % allow utf-8 input
\usepackage[T1]{fontenc}    % use 8-bit T1 fonts
\usepackage{url}            % simple URL typesetting
\usepackage{booktabs}       % professional-quality tables
\usepackage{amsfonts}       % blackboard math symbols
\usepackage{nicefrac}       % compact symbols for 1/2, etc.
\usepackage{microtype}      % microtypography
\usepackage{setspace}
\usepackage{amssymb,amsmath,amsthm,bbm}
\usepackage{verbatim,float,url,dsfont}
\usepackage{graphicx,subfigure,psfrag}
\usepackage{grffile}
\usepackage{algorithm,algorithmic}
\usepackage{mathtools,enumitem}
\usepackage{microtype}
\newtheorem{theorem}{Theorem}
\newtheorem{lemma}{Lemma}

\newtheorem{proposition}{Proposition}
\theoremstyle{definition}
\newtheorem{remark}{Remark}

\newtheorem*{assumption*}{\assumptionnumber}
\providecommand{\assumptionnumber}{}


\newcommand{\minimize}{\mathop{\mathrm{minimize}}}

\def\R{\mathbb{R}}
\def\E{\mathbb{E}}

\def\tr{\mathrm{tr}\,}

\def\Cov{\mathrm{Cov}}
\def\Var{\mathrm{Var}}

\def\col{\mathrm{col}}

\def\nul{\mathrm{null}}

\def\diag{\mathrm{diag}}

\def\hbeta{\hat{\beta}}

\def\ie{i.e.}
\def\eg{e.g.}
\def\cf{cf.}
\def\hSigma{\hat\Sigma}

\def\ridge{\mathrm{ridge}}
\def\gf{\mathrm{gf}}

\def\sgf{\mathrm{sgf}}

\def\Bias{\mathrm{Bias}}
\def\Risk{\mathrm{Risk}}

\usepackage{pifont,xcolor}

\def\ep{\epsilon}

\newcommand{\bitem}{\begin{itemize}}
\newcommand{\eitem}{\end{itemize}}
\newcommand{\benum}{\begin{enumerate}}
\newcommand{\eenum}{\end{enumerate}}
\newcommand{\beq}{\begin{equation}}
\newcommand{\eeq}{\end{equation}}
\newcommand{\beqs}{\begin{equation*}}
\newcommand{\eeqs}{\end{equation*}}
\newcommand{\bt}{\beta(t)}

\def\E{\mathbb{E}}
\newcommand{\inn}{\textrm{in}}

\usepackage[accepted,nohyperref]{icml2020}

\icmltitlerunning{The Implicit Regularization of Stochastic Gradient Flow for Least Squares}
\begin{document}

\twocolumn[
%\icmltitle{Statistical Learning with Stochastic Gradient Flow}
\icmltitle{The Implicit Regularization of Stochastic Gradient Flow for Least Squares}

% It is OKAY to include author information, even for blind
% submissions: the style file will automatically remove it for you
% unless you've provided the [accepted] option to the icml2020
% package.

% List of affiliations: The first argument should be a (short)
% identifier you will use later to specify author affiliations
% Academic affiliations should list Department, University, City, Region, Country
% Industry affiliations should list Company, City, Region, Country

% You can specify symbols, otherwise they are numbered in order.
% Ideally, you should not use this facility. Affiliations will be numbered
% in order of appearance and this is the preferred way.
%\icmlsetsymbol{equal}{*}

\begin{icmlauthorlist}
\icmlauthor{Alnur Ali}{stanford}
\icmlauthor{Edgar Dobriban}{penn}
\icmlauthor{Ryan J.~Tibshirani}{cmu}
\end{icmlauthorlist}

\icmlaffiliation{stanford}{Stanford University}
\icmlaffiliation{penn}{University of Pennsylvania}
\icmlaffiliation{cmu}{Carnegie Mellon University}

\icmlcorrespondingauthor{Alnur Ali}{alnurali@stanford.edu}
\icmlcorrespondingauthor{Edgar Dobriban}{dobriban@wharton.upenn.edu}

% You may provide any keywords that you
% find helpful for describing your paper; these are used to populate
% the "keywords" metadata in the PDF but will not be shown in the document
\icmlkeywords{Machine Learning, ICML}

\vskip 0.3in
]

% this must go after the closing bracket ] following \twocolumn[ ...

% This command actually creates the footnote in the first column
% listing the affiliations and the copyright notice.
% The command takes one argument, which is text to display at the start of the footnote.
% The \icmlEqualContribution command is standard text for equal contribution.
% Remove it (just {}) if you do not need this facility.

\printAffiliationsAndNotice{}  % leave blank if no need to mention equal contribution
%\printAffiliationsAndNotice{\icmlEqualContribution} % otherwise use the standard text.

\begin{abstract}
We study the implicit regularization of mini-batch stochastic gradient descent, when applied to the fundamental problem of least squares regression.  We leverage a continuous-time stochastic differential equation having the same moments as stochastic gradient descent, which we call \textit{stochastic gradient flow}.  We give a bound on the excess risk of stochastic gradient flow at time $t$, over ridge regression with tuning parameter $\lambda = 1/t$.  The bound may be computed from explicit constants (\eg, the mini-batch size, step size, number of iterations), revealing precisely how these quantities drive the excess risk.  Numerical examples show the bound can be small, indicating a tight relationship between the two estimators.  We give a similar result relating the coefficients of stochastic gradient flow and ridge.  These results hold under no conditions on the data matrix $X$, and across the entire optimization path (not just at convergence).
\end{abstract}

\section{Introduction}
Stochastic gradient descent (SGD) is one of the most widely used optimization algorithms---given the sizes of modern data sets, its scalability and ease-of-implementation means that it is usually preferred to other methods, including gradient descent \citep{bottou1998online,bottou2003stochastic,zhang2004solving,bousquet2008tradeoffs, bottou2010large, bottou2016optimization}.

A recent line of work \citep{nacson2018stochastic,gunasekar2018characterizing,soudry2018implicit,suggala2018connecting,ali2018continuous,PBL19,ji2019implicit} has shown that the iterates generated by gradient descent, when applied to a loss without any explicit regularizer, possess a kind of implicit $\ell_2$ regularity.  Implicit regularization is useful because it suggests a computational shortcut: the iterates generated by sequential optimization algorithms may serve as cheap approximations to the more expensive solution paths associated with explictly regularized problems.  While a lot of the interest in implicit regularization is new, its origins can be traced back at least a couple of decades, with several authors noting the apparent connection between early-stopped gradient descent and $\ell_2$ regularization \citep{strand1974theory,morgan1989generalization,friedman2004gradient,ramsay2005parameter,yao2007early}.
%, followed by work on boosting \citep{buhlmann2003boosting,rosset2004boosting,zhang2005boosting,raskutti2014early,wei2017early}.
%by a flurry of 
%(\ie, the iterates from gradient descent bear a striking resemblance to the corresponding explicitly $\ell_2$-penalized estimates)

Thinking of SGD as a computationally cheap but noisy version of gradient descent, it is natural to ask: do the iterates generated by SGD also possess a kind of $\ell_2$ regularity?  Of course, the connection here may not be as clear as with gradient descent, since there should be a price to pay for the computational savings.
% afforded by SGD.

In this paper, we study the implicit regularization performed by mini-batch stochastic gradient descent with a constant step size, when applied to the fundamental problem of least squares regression.  We defer a proper review of related work until later on, but for now mention that constant step sizes are frequently analyzed \citep{bach2013non,defossez2014constant,dieuleveut2017bridging,jain2017markov,babichev2018constant}, and popular in practice, because of their simplicity.  We adopt a continuous-time point-of-view, following \citet{ali2018continuous}, and study a stochastic differential equation that we call \textit{stochastic gradient flow}.  A strength of the continuous-time perspective is that it facilitates a direct and precise comparison to $\ell_2$ regularization, across the entire optimization path---not just at convergence, as is done in much of the current work on implicit regularization.

\paragraph{Summary of Contributions.}
A summary of our contributions in this paper is as follows.
\begin{itemize}%[topsep=0pt,partopsep=0pt,itemsep=0pt, parsep=0pt]
\item  We give a bound on the excess risk of stochastic gradient flow at time $t$, over ridge regression with tuning parameter $\lambda = 1/t$, for all $t \geq 0$.  The bound decomposes into three terms.  The first term is the (scaled) variance of ridge.  The second and third terms both stem from the variance due to mini-batching, and may be made smaller by, \eg, increasing the mini-batch size and/or decreasing the step size.  The second term may be interpreted as the ``price of stochasticity'': it is nonnegative, but vanishes as time grows.  The third term is tied to the limiting optimization error of stochastic gradient flow: it is zero in the overparametrized (interpolating) regime \citep{bassily2018exponential}, but is positive otherwise, reflecting the fact that stochastic gradient flow with a constant step size fluctuates around the least squares solution as time grows.  The bound holds with no conditions on the data matrix $X$.  Numerically, the bound can be small, indicating a tight relationship between the two estimators.
%(estimation) 
%directly

%\item We show a similar result holds for in-sample risk.

\item Using the bound, we show through numerical examples that stochastic gradient flow, when stopped at a time that (optimally) balances its bias and variance, yields a solution attaining risk that is 1.0032 times that of the (optimally-stopped) ridge solution, in less time---indicating that stochastic gradient flow strikes a favorable computational-statistical trade-off.

\item We give a similar bound on the distance between the coefficients of stochastic gradient flow at time $t$, and those of ridge regression with tuning parameter $\lambda = 1/t$, which is also seen to be tight.
\end{itemize}

\paragraph{Outline.}  Next, we review related work.  Section \ref{sec:prelims} covers notation, and further motivates the continuous-time approach.  In Section \ref{sec:risk}, we present our bound on the excess risk of stochastic gradient flow over ridge regression.  In Section \ref{sec:coeffs}, we present a bound relating the coefficients of the two estimators.  Section \ref{sec:exps} gives numerical examples supporting our theory.  In Section \ref{sec:disc}, we conclude.

\paragraph{Related Work.}  \textit{Stochastic Gradient Descent.}  The statistical and computational properties of SGD have been studied intensely over the years, with work tracing back to \citet{robbins1951stochastic,fabian1968asymptotic,ruppert1988efficient,kushner2003stochastic,polyak1992acceleration,nemirovski2009robust}.  On the statistical side, a lot of the work has focused on delivering optimal error rates for SGD and its many variants, \eg, with averaging, either asymptotically \citep{robbins1951stochastic,fabian1968asymptotic,ruppert1988efficient,kushner2003stochastic,polyak1992acceleration,moulines2011non,toulis2017asymptotic,nemirovski2009robust}, or in finite samples \citep{cesa1996worst,zhang2004solving,ying2008online,cesa2006prediction,pillaud2018statistical,jain2018parallelizing}.
%and many others; see e.g, \citet{bottou2016optimization} for a review
%and/or various step size schemes

Notably, \citet{bach2013non,defossez2014constant,dieuleveut2017bridging,jain2017markov,babichev2018constant} studied SGD with a constant step size for least squares regression with averaging (obtaining optimal rates, which is not our focus).  Good references on inference and computation include \citet{fabian1968asymptotic,ruppert1988efficient,polyak1992acceleration,moulines2011non,chen2016statistical,toulis2017asymptotic} and \citet{recht2011hogwild,duchi2015asynchronous}, respectively.  \citet{mandt2015continuous,duvenaud2016early} interpreted SGD with a constant step size as doing Bayesian inference.  Many works have empirically investigated the generalization properties of SGD, mainly in the context of non-convex optimization \citep{jastrzkebski2017three,kleinberg2018alternative,zhang2018theory,jin2019stochastic,nakkiran2019sgd,saxe2019mathematical}.

\textit{Implicit Regularization.}  Nearly all of the work in implicit regularization thus far has examined the convergence points of gradient descent, and not the whole path, for specific convex \citep{nacson2018stochastic,gunasekar2018characterizing,soudry2018implicit,vaskevicius2019implicit} and non-convex \citep{li2017algorithmic,wilson2017marginal,gunasekar2017implicit,gunasekar2018implicit} problems.  Notable exceptions include \citet{rosasco2015learning,lin2016generalization,lin2017optimal,neu2018iterate}, who studied averaged SGD with a constant step size for least squares regression, arguing that the various algorithmic parameters (\ie, the step size, mini-batch size, number of iterations, etc.) perform a kind of implicit regularization, by inspecting the corresponding error rates.  A few works have investigated implicit regularization outside of optimization \citep{mahoney2011implementing,mahoney2012approximate,gleich2014anti,martin2018implicit}.

\textit{Stochastic Differential Equations.}  Several papers have studied the same stochastic differential equation that we do \citep{hu2017diffusion,feng2017semi,li2019stochastic,feng2019uniform}, but without the focus on implicit regularization and statistical learning.  Along these lines, somewhat related work can be found in the literature on Langevin dynamics \citep{geman1986diffusions,seung1992statistical,neal2011mcmc,welling2011bayesian,sato2014approximation,teh2016consistency,raginsky2017non,cheng2019quantitative}.

\section{Preliminaries}
\label{sec:prelims}

\subsection{Least Squares, Stochastic Gradient Descent, and Stochastic Gradient Flow}
Consider the usual least squares regression problem,
\begin{equation}
\minimize_{\beta \in \R^p} \; \frac{1}{2n} \| y - X \beta \|_2^2, \label{eq:ls}
\end{equation}
where $y \in \R^n$ is the response and $X \in \R^{n \times p}$ is the data matrix.  Mini-batch SGD applied to \eqref{eq:ls} is the iteration
\begin{align}
\beta^{(k)} & = \beta^{(k-1)} + \frac{\epsilon}{m} \cdot \sum_{i \in \mathcal I_k} (y_i - x_i^T \beta^{(k-1)}) x_i \notag \\
& = \beta^{(k-1)} + \frac{\epsilon}{m} \cdot X_{\mathcal I_k}^T (y_{\mathcal I_k} - X_{\mathcal I_k} \beta^{(k-1)}), \label{eq:sgd}
\end{align}
for $k=1,2,3,\ldots$, where $\epsilon > 0$ is a fixed step size, $m$ is the mini-batch size, and $\mathcal I_k \subseteq \{1,\ldots,n\}$ denotes the mini-batch on iteration $k$ with $| \mathcal I_k | = m$, for all $k$.  For simplicity, we assume the mini-batches are sampled with replacement; our results hold with minor modifications under sampling without replacement.  We assume the initialization $\beta^{(0)} = 0$.
%Here and throughout, w

Now, adding and subtracting the negative gradient of the loss in \eqref{eq:sgd} yields
\begin{align}
& \beta^{(k)} = \beta^{(k-1)} + \frac{\epsilon}{n} \cdot X^T (y - X \beta^{(k-1)}) \label{eq:sgd:2} \\
& + \epsilon \cdot \Bigg( \frac{1}{m} X_{\mathcal I_k}^T (y_{\mathcal I_k} - X_{\mathcal I_k} \beta^{(k-1)}) - \frac{1}{n} X^T (y - X \beta^{(k-1)}) \Bigg). \notag
\end{align}
This may be recognized as gradient descent, plus the deviation between the sample average of $m$ i.i.d.~random variables and their mean, which motivates the continuous-time dynamics (stochastic differential equation)
\begin{equation}
d \beta(t) = \frac{1}{n} X^T (y - X \beta(t)) \, dt + Q_\epsilon(\beta(t))^{1/2} \, dW(t), \label{eq:sgf:nonconstcov}
\end{equation}
with $\beta(0) = 0$.  Here, $W(t)$ is standard $p$-dimensional Brownian motion.  We denote the diffusion coefficient
\begin{equation}
Q_\epsilon(\beta) = \epsilon \cdot \Cov_{\mathcal I} \Bigg( \frac{1}{m} X_{\mathcal I}^T (y_{\mathcal I} - X_{\mathcal I} \beta) \Bigg), \label{eq:Q}
\end{equation}
where the randomness is due to $\mathcal I \subseteq \{1,\ldots,n\}$.  We call the diffusion process \eqref{eq:sgf:nonconstcov} \textit{stochastic gradient flow}.

At this point, it helps to recall the related work of \citet{ali2018continuous}, who studied \textit{gradient flow},
\begin{equation}
\dot \beta(t) = \frac{1}{n} X^T(y - X \beta(t)) dt, \quad \beta(0) = 0, \label{eq:gf}
\end{equation}
which is gradient descent for \eqref{eq:ls} with infinitesimal step sizes.  In what follows, we frequently use the solution to \eqref{eq:gf},
\begin{equation}
\hat \beta^\gf(t) = (X^T X)^+ \big( I - \exp(-t X^T X / n) \big) X^T y, \label{eq:gf:soln}
\end{equation}
where $\exp(A)$ and $A^+$ denote the matrix exponential and the Moore-Penrose pseudo-inverse of $A$, respectively.

Unlike gradient flow, the continuous-time flow \eqref{eq:sgf:nonconstcov} does not arise by taking limits of the discrete-time dynamics \eqref{eq:sgd}, and should instead be interpreted as an approximation to \eqref{eq:sgd}.  To see this, consider the Euler discretization of \eqref{eq:sgf:nonconstcov},
\begin{align}
& \tilde \beta^{(k)} = \tilde \beta^{(k-1)} + \frac{\epsilon}{n} \cdot X^T (y - X \tilde \beta^{(k-1)}) \label{eq:euler} \\
& \hspace{0.4in} + \epsilon \cdot \Cov_{\mathcal I}^{1/2} \Bigg( \frac{1}{m} X_{\mathcal I}^T (y_{\mathcal I} - X_{\mathcal I} \tilde \beta^{(k-1)}) \Bigg) z_k, \notag
\end{align}
where $z_k \sim N(0,I)$ and $\tilde \beta^{(0)} = 0$, \ie, \eqref{eq:euler} approximates \eqref{eq:sgd:2} with a Gaussian process.  Note that the noise in \eqref{eq:euler} is on the right scale, which also explains the presence of $\epsilon$ in \eqref{eq:Q}.

Figure \ref{fig:paths} presents a small numerical example, where we see a striking resemblance between the paths for SGD, the Euler discretization of stochastic gradient flow, and ridge regression with tuning parameter $\lambda = 1/t$.

\begin{figure*}[ht]
%\vskip 0.2in
\begin{center}
\centerline{
\includegraphics[width=0.33\textwidth]{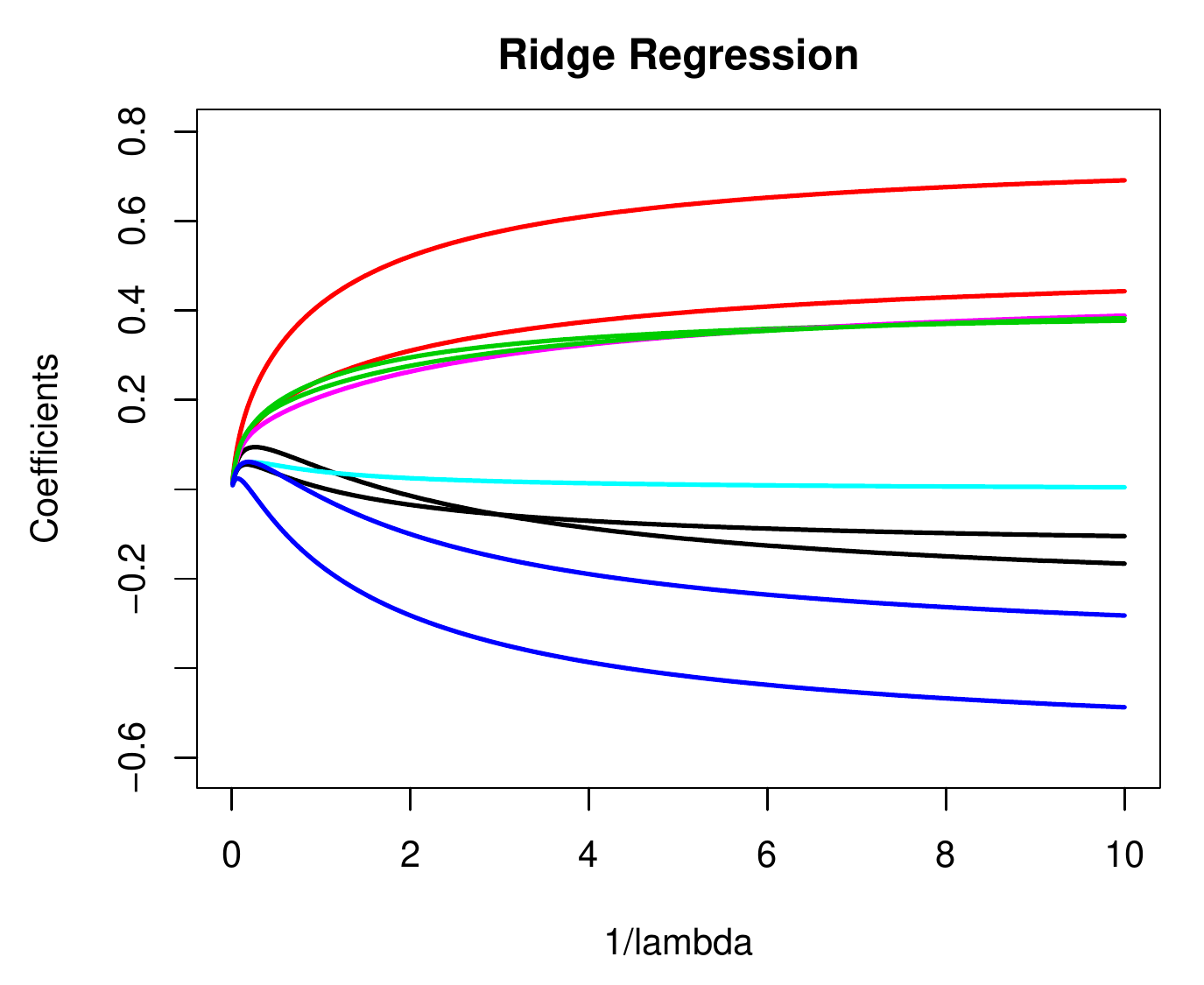} \hfill
\includegraphics[width=0.33\textwidth]{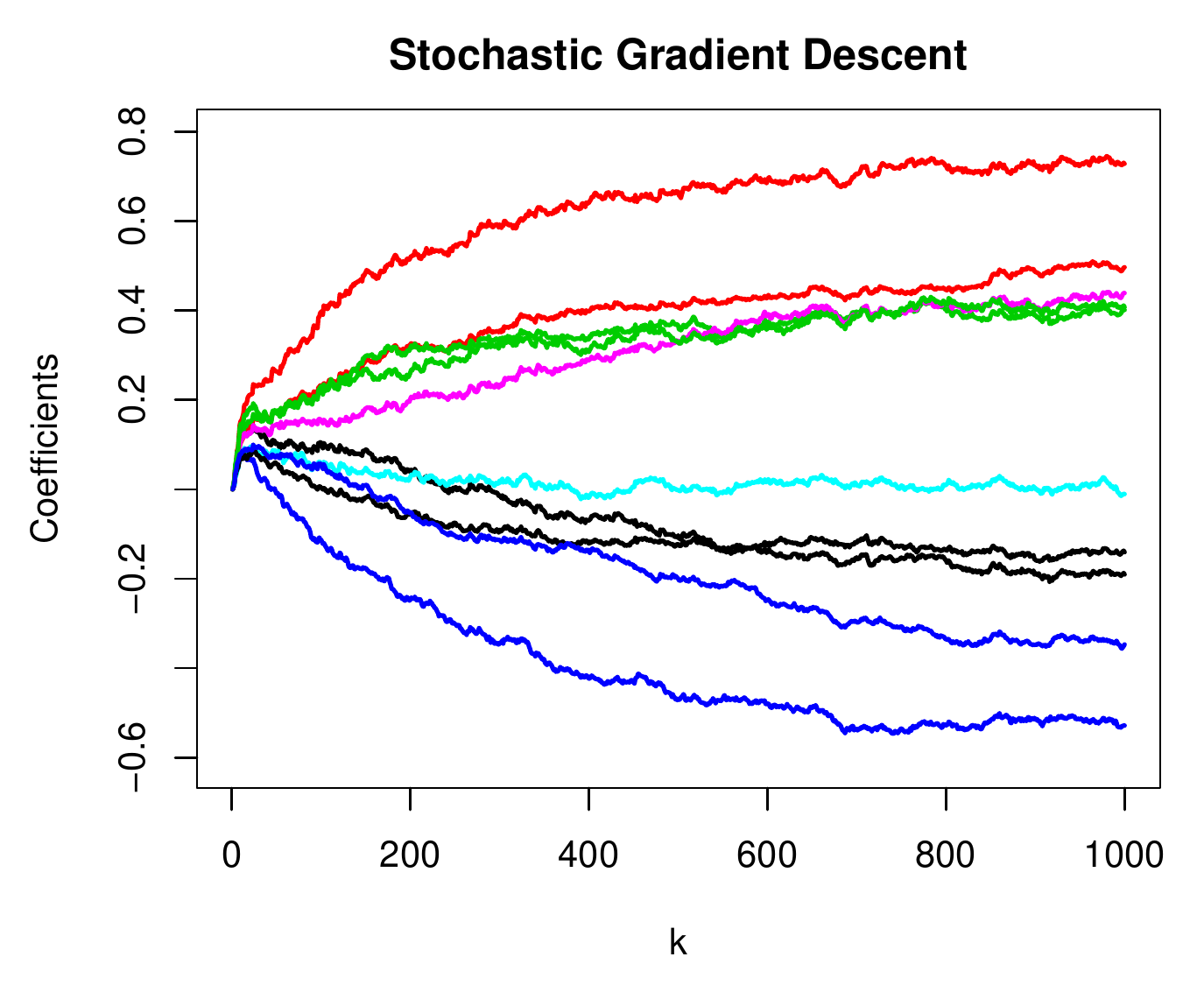} \hfill
\includegraphics[width=0.33\textwidth]{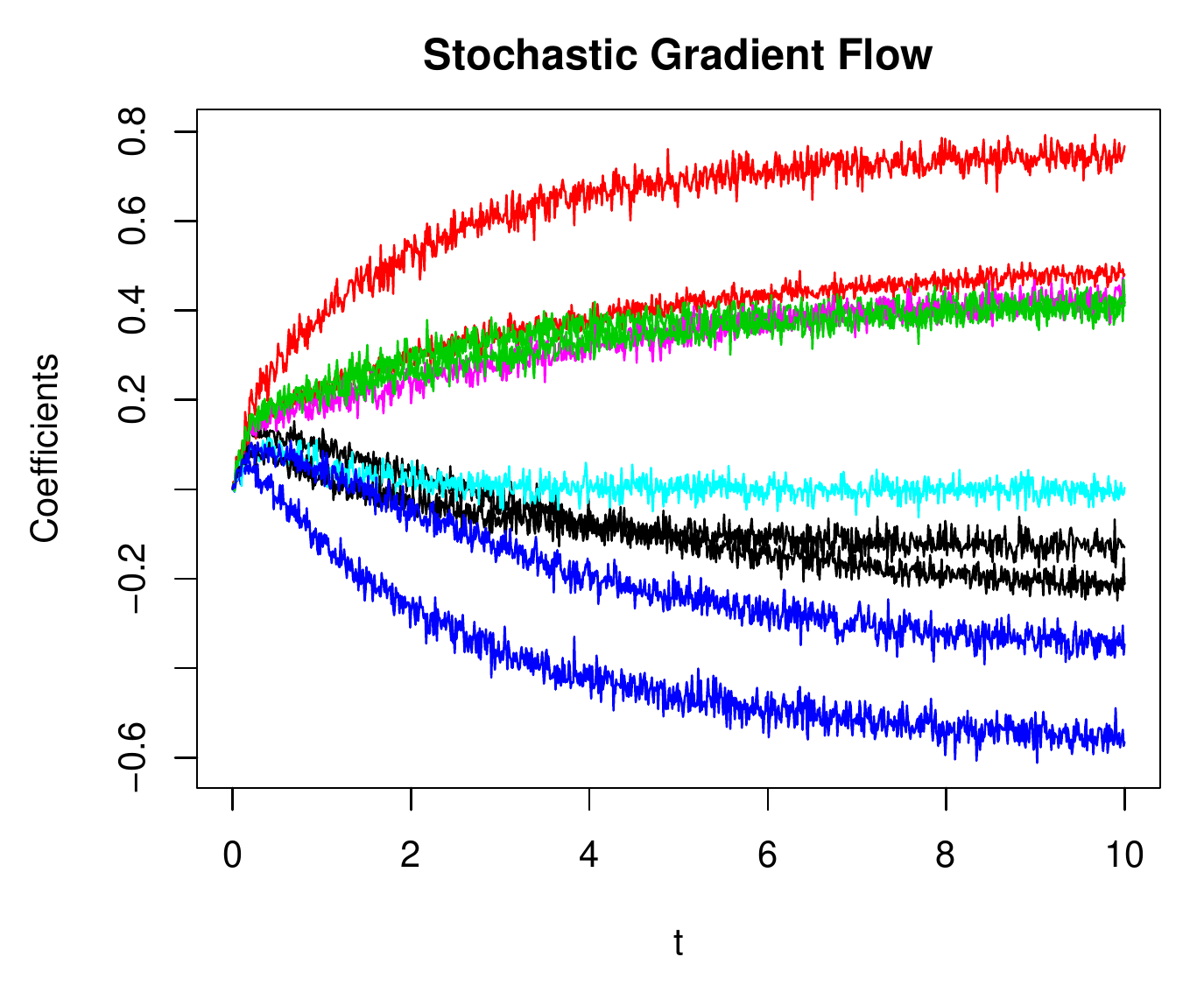}
}
\vskip -0.2in
\caption{\textit{Solution and optimization paths for ridge regression (left panel), SGD (middle panel), and the Euler discretization of stochastic gradient flow (right panel) on a small example, where $n=50$, $p=10$, $m=10$, and $\epsilon = 0.01$.}}
\label{fig:paths}
\end{center}
\vskip -0.2in
\end{figure*}

\subsection{Basic Properties of Stochastic Gradient Flow}
We begin with an important lemma further motivating the differential equation \eqref{eq:sgf:nonconstcov}; its proof, as with many of the results in this paper, may be found in the supplement.  The result shows that both the first and second moments of the Euler discretization of \eqref{eq:sgf:nonconstcov} match those of the underlying discrete-time SGD iteration.  This means that any deviation between the first two moments of the continuous-time flow \eqref{eq:sgf:nonconstcov} and discrete-time SGD must be due to discretization.
%error

%[Moment matching]
\begin{lemma} \label{mom-match}
Fix $y$, $X$, $\epsilon > 0$, and $k \geq 1$.  Write $\tilde \beta^{(k)}$ for the Euler discretization \eqref{eq:euler} of stochastic gradient flow, and $\beta^{(k)}$ for SGD (both using $\epsilon$).  Then, the first and second moments of $\tilde \beta^{(k)}$ match those of $\beta^{(k)}$, \ie, we have that both
\begin{itemize}
\item $\E_{\tilde Z} \tilde \beta^{(k)} = \E_{\mathcal I_1,\ldots,\mathcal I_k} \beta^{(k)}$, and
\item $\Cov_{\tilde Z} \tilde \beta^{(k)} = \Cov_{\mathcal I_1,\ldots,\mathcal I_k} \beta^{(k)}$.
\end{itemize}
Here, we let $\tilde Z$ denote the randomness inherent to $\tilde \beta^\sgf(t)$.
\end{lemma}

\begin{remark}
The result also implies that both the estimation and out-of-sample risks of $\tilde \beta^{(k)}$ match those of $\beta^{(k)}$; we defer a more thorough treatment of this point to Section \ref{sec:risk}.
\end{remark}
%also 

\begin{remark}
Discretization, \ie, showing that \eqref{eq:euler} and \eqref{eq:sgd} are close in a precise sense, turns out to be non-trivial, and is left to future work.
\end{remark}
%Further motivation for \eqref{eq:sgf:nonconstcov} is given later on.

Next, with the above motivation in mind, we present a lemma establishing that the solution to \eqref{eq:sgf:nonconstcov} exists and is unique.  The result also gives a more explicit expression for the solution to \eqref{eq:sgf:nonconstcov}, which plays a key role in many of the results to come.

%[Integral representation of SGF] 
\begin{lemma} \label{lem:soln}
Fix $y$, $X$, and $\epsilon > 0$.  Let $t \in [0,T]$.  Then
\begin{align}
& \hat \beta^\sgf(t) = \hat \beta^\gf(t) \label{eq:soln} \\
& \hspace{0.2in} + \exp ( - t  \hSigma ) \cdot \int_0^t \exp ( \tau  \hSigma ) Q_\epsilon(\hat \beta^\sgf(\tau))^{1/2} d W(\tau) \notag
\end{align}
is the unique solution to the differential equation \eqref{eq:sgf:nonconstcov}.  
\end{lemma}

\begin{remark}
The result actually holds for any Lipschitz continuous diffusion coefficient $Q_\epsilon(\beta(t))$, \eg, $Q_\epsilon(\beta(t)) = I$, as well as the time-homogeneous covariance $Q_\epsilon(\beta(t)) = (\epsilon/m) \cdot \hat \Sigma$ \citep{mandt2017stochastic,wang2017asymptotic,dieuleveut2017harder,fan2018statistical}. In the former case, \eqref{eq:sgf:nonconstcov} reduces to (rescaled) Langevin dynamics.
\end{remark}
%\citep{geman1986diffusions,seung1992statistical,neal2011mcmc,welling2011bayesian,sato2014approximation,teh2016consistency,raginsky2017non,cheng2019quantitative}

\subsection{Constant vs.~Non-Constant Covariances}
The differential equation \eqref{eq:sgf:nonconstcov} has been considered previously \citep{hu2017diffusion,feng2017semi,li2019stochastic,feng2019uniform}, but several works \citep{mandt2017stochastic,wang2017asymptotic,dieuleveut2017harder,fan2018statistical} have found it convenient to work with the simplification
\begin{equation}
d \beta(t) = \frac{1}{n} X^T (y - X \beta(t)) \, dt + \Big( \frac{\epsilon}{m} \cdot \hSigma \Big)^{1/2} \, dW(t), \label{eq:sgf:constcov}
\end{equation}
where $\beta({0}) = 0$.  Here, $Q_\epsilon(\beta(t)) = (\epsilon/m) \cdot \hat \Sigma$.  However, we present a simple but telling example revealing that these two processes, \ie, the non-constant covariance process in \eqref{eq:sgf:nonconstcov}, and the constant covariance process in \eqref{eq:sgf:constcov}, need not be close in general.

Consider the univariate responseless least squares problem,
\[
\minimize_{\beta \in \R} \; \frac{1}{2n} \sum_{i=1}^n (x_i \beta)^2.
\]
Let $G_k = (1/m) \sum_{i \in \mathcal I_k} x_i^2$, for $k=1,2,3,\ldots$.  Then SGD for the above problem may be expressed as
\[
\beta^{(k)} = \beta^{(k-1)} - \epsilon \cdot G_k \beta^{(k-1)} = \prod_{j=1}^{k-1} \Bigg( 1 - \epsilon \cdot G_j \Bigg) \beta^{(0)}.
\]
Assume the initial point is a nonzero constant, the $x_i$ follow a continuous distribution, and $\epsilon$ is sufficiently small.  Letting $t > 0$ be arbitrary, the basic estimate $1-x \leq \exp(-x)$ combined with Markov's inequality shows that
%, $i=1,\ldots,n$, 
\[
\Pr( \beta^{(k)} > t) \leq \E \Bigg[ \exp \Bigg( - \epsilon \cdot \sum_{j=1}^{k-1} G_j \Bigg) \Bigg] \beta^{(0)} / t.
\]
Summing the right-hand side over $k=1,\ldots,\infty$, we conclude that $\beta^{(k)}$ converges to zero with probability one, by the first Borel-Cantelli lemma.

Now let $G = (1/n) \sum_{i=1}^n x_i^2$.  We may calculate for the non-constant process that $Q_\epsilon(\beta(t))^{1/2} = \theta \beta(t)$, where $\theta = (\epsilon/m \cdot G)^{1/2}$, meaning the non-constant process follows the dynamics (the sign of $Q_\epsilon^{1/2}$ may be chosen arbitrarily) 
\[
d \beta(t) = - G \beta(t) + \theta \beta(t) dW(t),
\]
which may be recognized as a geometric Brownian motion.  It can be checked that both the mean and variance of the geometric Brownian motion tend to zero as time grows, provided that $\theta^2 < 2 G$, which certainly holds when $\epsilon < 1$.

On the other hand, the constant process is an Ornstein-Uhlenbeck process,
\[
d \beta(t) = - G \beta(t) + \theta dW(t).
\]
Again, it may be checked (\eg, Chapter 5 in \citet{oksendal2003stochastic}) that the process mean goes down to zero, whereas the variance tends to the constant $\epsilon/(2m)$.  In other words, the limiting dynamics of the constant process exhibit constant-order fluctuations, whereas those of the non-constant process do not. Therefore, for this problem, the latter dynamics more accurately reflect those of discrete-time SGD.  See Figure \ref{fig:contour} for an example. %\ed{add details} %illustrates that the non-constant process can more accurately reflect the underlying dynamics of discrete-time SGD than the constant process.

\begin{figure}[ht]
\vskip -0.2in
\begin{center}
\centerline{\includegraphics[width=1.2\columnwidth]{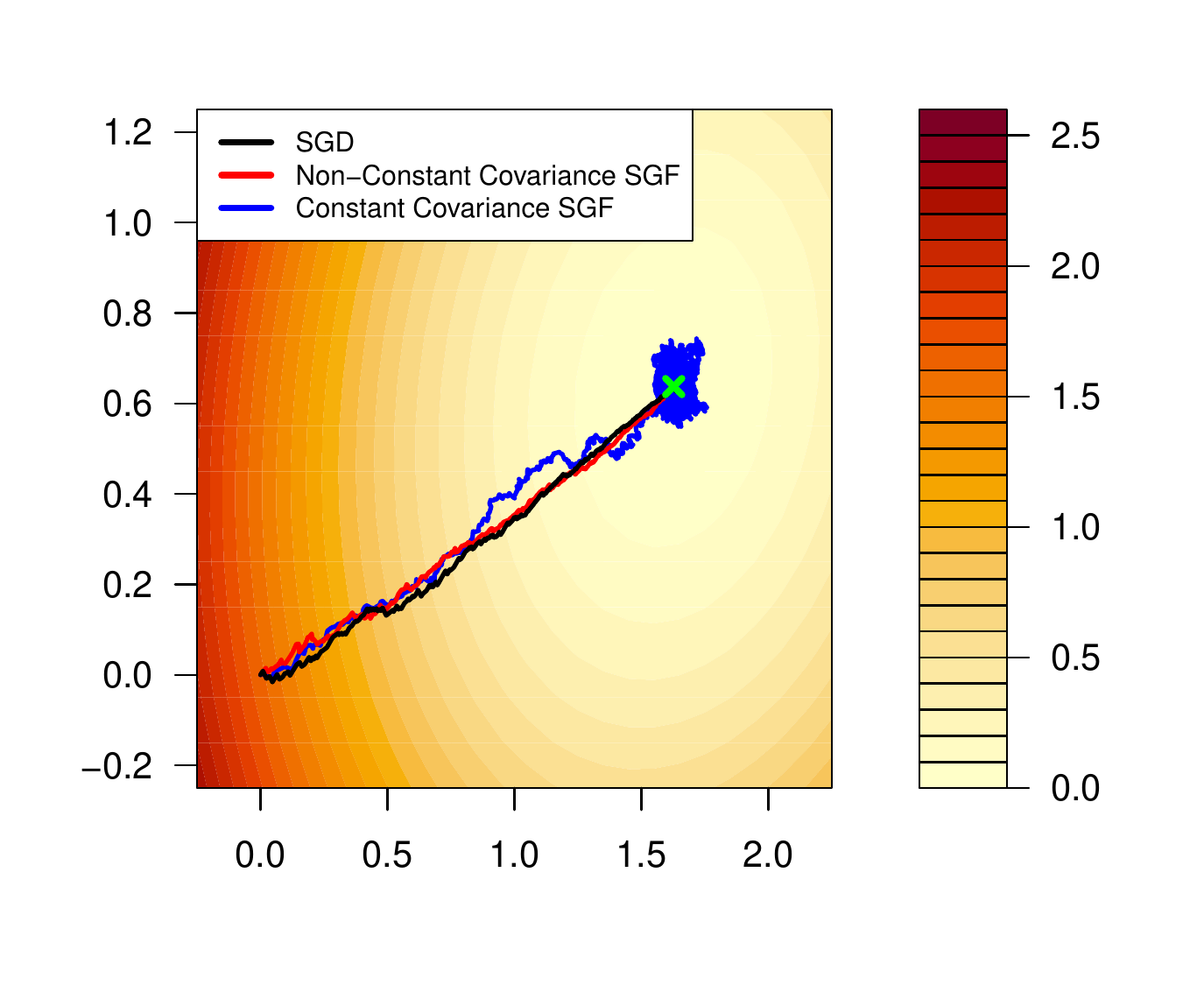}}
\vskip -0.4in
\caption{\textit{Trajectories for the non-constant and constant covariance processes, as well as discrete-time SGD, on a simple least squares problem, where $n=3$, $p=2$, $m=2$, and $\epsilon = 0.01$.  Warmer colors denote larger values of the least squares loss function, and the green X denotes the least squares solution.}}
\label{fig:contour}
\end{center}
\vskip -0.2in
\end{figure}

We close this section with a simple result bounding the deviation between solutions to the non-constant and constant processes, in expectation.  The result indicates that the two processes can be close when the non-constant process dynamics are close to the underlying coefficients.  A thorough comparison of the two processes is left to future work.

%[Conditions for closeness of constant and non-constant processes] 
\begin{lemma} \label{lem:nonconst_vs_const}
Fix $y$, $X$, and $\epsilon > 0$.  Let $t \geq 0$.  Write $\tilde \beta^\sgf(t)$ for the solution to the constant process.  Then
\begin{align*}
& \E_{Z, \tilde Z} \| \hat \beta^\sgf(t) - \tilde \beta^\sgf(t) \|_2^2 \leq 4 L p^3 \epsilon / m \\
& \times \int_0^t \E_Z \Big[ \sum_{i=1}^n \big| (y_i - x_i^T \hat \beta^\sgf(\tau))^2 - 1 \big| \Big] d \tau.
\end{align*}
Here, we let $Z$, $\tilde Z$ denote the randomness inherent to $\hat \beta^\sgf(t)$, $\tilde \beta^\sgf(t)$, respectively, and write $L = \lambda_{\max}(\hat \Sigma)$.
\end{lemma}

\section{Statistical Risk Bounds}
\label{sec:risk}

\subsection{Measures of Risk and Notation}
Here and throughout, we let the predictor matrix $X$ be arbitrary and fixed, and assume the response $y$ follows a standard regression model,
\[
y = X \beta_0 + \eta,
\]
for some fixed underlying coefficients $\beta_0 \in \R^p$, and noise $\eta \sim (0, \sigma^2 I)$.  We consider the statistical (estimation) risk of an estimator $\hat \beta \in \R^p$,
\[
\Risk(\hat \beta; \beta_0) = \E_{\eta,  Z} \| \hat \beta - \beta_0 \|_2^2.
\]
Here $Z$ denotes any potential randomness inherent to $\hat \beta$ (\eg, due to mini-batching).  We also consider in-sample risk,
\[
\Risk^\inn(\hat \beta; \beta_0) = \frac{1}{n} \E_{\eta,  Z} \| X \hat \beta - X \beta_0 \|_2^2.
\]

We let $\hat \Sigma = X^T X / n$ denote the sample covariance matrix with eigenvalues $s_i$ and eigenvectors $v_i$, for $i=1,\ldots,p$, and let $\mu = \min_i s_i$ and $L = \max_i s_i$ denote the smallest nonzero and largest eigenvalues of $\hat \Sigma$, respectively.

\subsection{Risk Bounds}
Recall the bias-variance decomposition for risk,
\begin{align*}
& \Risk(\hat \beta^\sgf(t); \beta_0) \notag \\
& = \| \E_{\eta,  Z} (\hat \beta^\sgf(t)) - \beta_0 \|_2^2 + \tr \Cov_{\eta, Z} (\hat \beta^\sgf(t)) \notag \\
& = \Bias^2(\hat \beta^\sgf(t); \beta_0) + \Var_{\eta, Z}(\hat \beta^\sgf(t)). \notag
\end{align*}
A straightforward calculation using the law of total variance shows (see the proof of Theorem \ref{thm:risk} for details)
\begin{align*}
& \Bias^2(\hat \beta^\sgf(t); \beta_0) = \Bias^2(\hat \beta^\gf; \beta_0) \\
& \Var_{\eta, Z}(\hat \beta^\sgf(t)) = \tr \E_\eta \big[ \Cov_Z(\hat \beta^\sgf(t)) \,|\, \eta)  \big] + \Var_\eta( \hat \beta^\gf(t)).
\end{align*}
Therefore, for stochastic gradient flow, the randomness due to mini-batching contributes to the estimation variance.  Hence, a tight bound on the variance due to mini-batching, $\tr \E_\eta \big[ \Cov_Z(\hat \beta^\sgf(t)) \,|\, \eta) \big]$, leads to a tight bound on the risk.  The following result, which we see as one of the main technical contributions of this paper, delivers such a bound.

%[Bound on mini-batching variance]
\begin{lemma} \label{thm:coeffs:second_term}
Fix $y$, $X$, and $\epsilon > 0$.  Let $t > 0$.  Then
\begin{align*}
& \tr \Cov_Z(\hat \beta^\sgf(t)) \\
& \quad \leq \frac{2 n \epsilon}{m} \cdot \int_0^t f(\hat \beta^\sgf(\tau)) \tr [ \hat \Sigma \exp(2 (\tau - t) \hat \Sigma) ] d \tau,
\end{align*}
where $f(\hat \beta^\sgf(\tau)) = \E_Z \Big[ {(2n)}^{-1} \| y - X \hat \beta^\sgf(\tau) \|_2^2 \Big]$.
\end{lemma}

\begin{remark}
The proof of the result depends critically on the special covariance structure of the diffusion coefficient, $Q_\epsilon(\beta(\tau))$, arising in the context of least squares regression.  To be more specific, for a fixed $\beta$, let $h(\beta) = ( y_1 - x_1^T \beta, \ldots, y_n - x_n^T \beta )$ denote the residuals at $\beta$, $F(\beta) = \diag( h(\beta) )^2$, and $\tilde F(\beta) = n^{-1} h(\beta) h(\beta)^T$.  Then, another calculation shows (\cf~\citet{hoffer2017train,zhang2017determinantal,hu2017diffusion})
\begin{align*}
Q_\epsilon(\beta) & = \Cov_{\mathcal I} \Bigg( \frac{1}{m} X_{\mathcal I}^T(y_{\mathcal I} - X_{\mathcal I} \beta) \Bigg) \\
& = \frac{1}{nm} X^T (F(\beta) - \tilde F(\beta)) X \\
& \preceq \frac{1}{nm} X^T F(\beta) X,
\end{align*}
which may be manipulated to obtain the result given in the lemma (see the supplement for details).
\end{remark}

\begin{remark}
As we discuss later in Section \ref{sec:coeffs}, the bound on the variance due to mini-batching given in Lemma \ref{thm:coeffs:second_term}, turns out to be central: it may also be used to give a tight bound on the coefficient error, $\E_{\eta,Z} \| \hat \beta^\sgf(t) - \hat \beta^\ridge(1/t) \|_2^2$.
\end{remark}

Inspecting the bound in Lemma \ref{thm:coeffs:second_term}, we see that the variance due to mini-batching, $\tr \E_\eta \big[ \Cov_Z(\hat \beta^\sgf(t)) \,|\, \eta)  \big]$, depends on the expected loss of stochastic gradient flow, $f(\hat \beta^\sgf(t))$.  It is reasonable to expect that stochastic gradient flow converges linearly, by analogy to the results that are available for SGD \citep{karimi2016linear,vaswani2018fast,bassily2018exponential}.  The following lemma gives the details.

%[Decay of the loss]
\begin{lemma} \label{dyna:exp:loss}
Fix $y$ and $X$.  Let $t \geq 0$.
\begin{itemize}%[topsep=0pt,partopsep=0pt,itemsep=0pt, parsep=0pt]
\item If $n > p$, define 
%\item If $X$ is full column rank (\ie, $n \geq p$), then set 
\begin{align*}
u & = \mu/n - 2 (n \epsilon)^2 / m \cdot \Big( \max_{i=1,\ldots,p} \big[ \diag(\hat \Sigma^2) \big]_i \Big) \\
v & = \Big[ 2 (n \epsilon)^2 q / m \cdot 
\max_{i=1,\ldots,p} \big[ \diag(\hat \Sigma^2) \big]_i \Big] / u \\
w & = \log \Big( \| y \|_2^2 / (2n) - q / (2n) \Big).
\end{align*}
Here, $q = \| P_{\nul(X^T)} y \|_2^2$ denotes the squared norm of the projection of $y$ onto the orthocomplement of the column space of $X$.

\item If $p \geq n$, define
\begin{align*}
u & = \mu/n - (n \epsilon)^2 / m \cdot \Big( \max_{i=1,\ldots,p} \big[ \diag(\hat \Sigma^2) \big]_i \Big) \\
v & = 0 \\
w & = \log \Big( \| y \|_2^2 / (2n) \Big).
\end{align*}
\end{itemize}
In either case, set $\epsilon$ small enough so that $u > 0$.  Then,
\[
f(\hat \beta^\sgf(t)) \leq \exp(-u t + w ) + v.
\]
\end{lemma}

\begin{remark}
Lemma \ref{dyna:exp:loss} can be seen as the continuous-time analog of, \eg, Theorem 4 in \citet{karimi2016linear}, and may be of independent interest.
\end{remark}
%Lemma \ref{dyna:exp:loss} can be seen as the continuous-time analog of, \eg, Theorem 1 in \citet{bassily2018exponential}, and may be of independent interest.
%Theorems 4 and 1 in \citet{karimi2016linear} and 
% respectively,
%The result shows that the risk of stochastic gradient flow converges exponentially fast, close to the least squares solution, 

Now define $\tilde w = \E_\eta [ \exp(w) ]$, $\tilde v = \E_\eta (v)$, as well as the effective variance due to mini-batching terms,
\begin{align}
& \nu_i(t) = \frac{\exp(w) s_i}{s_i - u/2} \big( \exp(-u t) - \exp(-2 t s_i) \big) \label{eq:eff} \\
& \hspace{1in} + v \big( 1 - \exp(-2 t s_i) \big), \; i=1,\ldots,p. \notag
\end{align}

We recall a result from \citet{ali2018continuous}, paraphrased below.
%At this point, it helps to recall a related result from \citet{ali2018continuous}, paraphrased below.
%prior work of 

\begin{theorem}[Theorem 1 in \citet{ali2018continuous}] \label{thm:gf}
Fix $X$.  Let $t \geq 0$.  Write $\hat \beta^\ridge(\lambda) = (X^T X + n \lambda I)^{-1} X^T y$, for the ridge regression estimate with tuning parameter $\lambda \geq 0$.  Then, $\Bias^2(\hat \beta^\gf(t); \beta_0) \leq \Bias^2(\hat \beta^\ridge(1/t); \beta_0)$, and $\Var(\hat \beta^\gf(t)) \leq 1.6862 \cdot \Var(\hat \beta^\ridge(1/t))$, so that $\Risk(\hat \beta^\gf(t); \beta_0) \leq 1.6862 \cdot \Risk(\hat \beta^\ridge(1/t); \beta_0)$.
\end{theorem}

Putting Lemmas \ref{thm:coeffs:second_term} and \ref{dyna:exp:loss} together with Theorem \ref{thm:gf} yields the following result, relating the risk of stochastic gradient flow to that of gradient flow and ridge regression.  
% with tuning parameter $\lambda = 1/t$

%[Excess risk bound for SGF]
\begin{theorem} \label{thm:risk}
Fix $X$.  Set $\epsilon$ according to Lemma \ref{dyna:exp:loss}.  Let $t > 0$.
\begin{itemize}[topsep=0pt,partopsep=0pt,itemsep=0pt, parsep=0pt]
\item Then, relative to gradient flow,
\begin{align}\label{eq:10}
& \Risk(\hat \beta^\sgf(t); \beta_0) 
 \leq \Bias^2(\hat \beta^\gf(t); \beta_0) \\
 &+ \Var_\eta( \hat \beta^\gf(t))
 + \epsilon \cdot \frac{n}{m} \sum_{i=1}^p \E_\eta \nu_i(t). \notag
\end{align}

\item Relative to ridge regression,
\begin{align}\label{eq:10b}
& \Risk(\hat \beta^\sgf(t); \beta_0) 
 \leq \Bias^2(\hat \beta^\ridge(1/t); \beta_0) \\
 &+ 1.6862 \cdot \Var_\eta( \hat \beta^\ridge(1/t))
 + \epsilon \cdot \frac{n}{m} \sum_{i=1}^p \E_\eta \nu_i(t). \notag
\end{align}
\end{itemize}
\end{theorem}

The analogous results for in-sample risk are similar, and deferred to the supplement for space reasons.

\begin{proof}
From Lemma \ref{lem:soln}, we have
\begin{align*}
\hat \beta^\sgf(t) = \hat \beta^\gf(t) + \int_0^t \exp \big( (\tau-t) \hat \Sigma) \big) Q_\epsilon(\beta(\tau))^{1/2} dW(\tau).
\end{align*}
The law of total expectation coupled with standard properties of Brownian motion (\eg, Theorem 3.2.1 in \citet{oksendal2003stochastic}) implies $\E_{\eta, Z} ( \hat \beta^\sgf(t) ) = \E_\eta \big[ \E_Z ( \hat \beta^\sgf(t) \,|\, \eta ) \big] = \E_\eta ( \hat \beta^\gf(t) )$.  Therefore,
\begin{equation}
\Bias^2(\hat \beta^\sgf(t); \beta_0) = \Bias^2(\hat \beta^\gf(t); \beta_0). \label{eq:3}
\end{equation}

Turning to the variance, the law of total variance and the above calculation implies
\begin{align}
& \tr \Cov_{\eta, Z}(\hat \beta^\sgf(t)) \notag \\
& = \tr \Big( \E_\eta \big[ \Cov_Z( \hat \beta^\sgf(t) \,|\, \eta ) \big] + \Cov_\eta \big( \E_Z ( \hat \beta^\sgf(t) \,|\, \eta ) \big) \Big) \notag \\
& = \tr \Big( \E_\eta \big[ \Cov_Z( \hat \beta^\sgf(t) \,|\, \eta ) \big] + \Cov_\eta(  \hat \beta^\gf(t) ) \Big) \notag \\
& = \tr \E_\eta \big[ \Cov_Z( \hat \beta^\sgf(t) \,|\, \eta ) \big] 
+ \Var_\eta( \hat \beta^\gf(t)). \label{eq:4a}
\end{align}

As for the trace appearing in \eqref{eq:4a}, we have
\begin{align*}
& \tr \E_\eta \big[ \Cov_Z ( \hat \beta^\sgf(t) \,|\, \eta ) \big] \notag 
 = \E_\eta \big[ \tr \Cov_Z ( \hat \beta^\sgf(t) \,|\, \eta ) \big] \notag \\
& \leq \E_\eta \Bigg[ \frac{2 n \epsilon}{m} \cdot \int_0^t f(\hat \beta^\sgf(\tau)) \tr [ \hat \Sigma \exp(2 (\tau - t) \hat \Sigma) ] d \tau \Bigg] \notag
\end{align*}
\begin{align}
& = \frac{2 n \epsilon}{m} \cdot \int_0^t \E_\eta [ f(\hat \beta^\sgf(\tau)) ] \tr [ \hat \Sigma \exp(2 (\tau - t) \hat \Sigma) ] d \tau \label{eq:4c} \\
& \leq \epsilon \cdot \frac{n}{m} \sum_{i=1}^p \Big( \tilde v \big( 1 - \exp(-2 t s_i) \big) \notag \\ %\label{eq:4d}
& \quad + \frac{\tilde w s_i}{s_i - u/2} \big( \exp(-u t) - \exp(-2 t s_i) \big) \Big). \label{aa}
\end{align}
Here, the second line followed from Lemma \ref{thm:coeffs:second_term}.  The third followed from Fubini's theorem.  The fourth followed by integrating, using the eigendecomposition $\hat \Sigma = V S V^T$ and Lemma \ref{dyna:exp:loss}, along with one final application of Fubini's theorem.  This shows the claim for gradient flow.  The claim for ridge follows by applying Theorem \ref{thm:gf}.
\end{proof}

The following result gives a more interpretable version of Theorem \ref{thm:risk}, at the expense of some sharpness.

%[Simplified excess risk bound for SGF]
\begin{lemma} \label{thm:relative_risk}
Fix $X$.  Set $\epsilon$ as in Lemma \ref{dyna:exp:loss}.  Let $t > 0$.  Define
\begin{align*}
\alpha & = p \tilde w \epsilon \cdot \frac{n \mu}{m (\mu - u/2)}, \\
\gamma(t) & = 1 + 2.164 \epsilon \cdot \frac{ \tilde v n^2 \max (1/t, L) }{m \sigma^2},
\end{align*}
$\kappa=L/\mu$, and $\delta=\alpha / {\|\beta_0\|_2}^{1/\kappa}$.
\begin{itemize}[topsep=0pt,partopsep=0pt,itemsep=0pt, parsep=0pt]
\item Then, for gradient flow,
\begin{align*}
& \Risk(\hat \beta^\sgf(t); \beta_0) \leq \Bias^2(\hat \beta^\gf(t); \beta_0) \\
& + \delta \cdot | \Bias(\hat \beta^\gf(t); \beta_0) |^{1/\kappa} + \gamma(t) \cdot \Var(\hat \beta^\gf(t)).
\end{align*}

\item For ridge regression,
\begin{align*}
& \Risk(\hat \beta^\sgf(t); \beta_0) \leq \Bias^2(\hat \beta^\ridge(1/t); \beta_0) \\
& \hspace{0.2in} + \delta \cdot | \Bias(\hat \beta^\ridge(1/t); \beta_0) |^{1/\kappa} \\
& \hspace{0.2in} + 1.6862 \gamma(t) \cdot \Var(\hat \beta^\ridge(1/t)).
\end{align*}
\end{itemize}
\end{lemma}

\begin{remark}
Interestingly, the result shows that the risk of stochastic gradient flow may be seen as the ridge bias raised to a power strictly less than 1, plus a time-dependent scaling of the ridge variance---which is quite different from the situation with gradient flow (\cf~Theorem \ref{thm:gf}).
\end{remark}

Finally, subtracting the ridge risk from both sides of \eqref{eq:10b} immediately gives our main result, a bound on the excess risk of stochastic gradient flow over ridge.

\begin{theorem} \label{thm:excess}
Fix $X$.  Set $\epsilon$ as in Lemma \ref{dyna:exp:loss}.  Let $t > 0$.  Then,
\begin{align}
& \Risk(\hat \beta^\sgf(t); \beta_0) - \Risk(\hat \beta^\ridge(1/t); \beta_0) \label{eq:excess} \\
& \leq 0.6862 \cdot \Var_\eta( \hat \beta^\ridge(1/t))
 + \epsilon \cdot \frac{n}{m} \sum_{i=1}^p \E_\eta \nu_i(t). \notag
\end{align}
\end{theorem}

%\begin{remark}
We can understand the influence of the effective variance terms on the risks \eqref{eq:10}, \eqref{eq:10b}, \eqref{eq:excess} as follows.  As stochastic gradient flow moves away from initialization, the stochastic gradients become smaller, and so their variance decreases, which is captured by the first term in \eqref{eq:eff}, as it goes down with time.  As stochastic gradient flow approaches the least squares solution, there are two possibilities, depending on whether the solution is interpolating or not.  If the solution is interpolating, then stochastic gradient flow can fit the data perfectly, and hence $v = 0$ in \eqref{eq:eff}.  Otherwise, stochastic gradient flow fluctuates around the solution, which is captured by the second term in \eqref{eq:eff}, as it grows with time.
%\end{remark}

%\begin{remark}
It is also interesting to note that the bounds \eqref{eq:10}, \eqref{eq:10b}, \eqref{eq:excess} depend linearly on $\epsilon/m$, corroborating recent empirical work \citep{krizhevsky2014one,goyal2017accurate,smith2017don,you2017scaling,shallue2019measuring}.
%\end{remark}

\begin{remark} \label{rem:const}
For space reasons, we compare the excess risk bound \eqref{eq:excess} to the analogous bound for the time-homogeneous process \eqref{eq:sgf:constcov} in the supplement.
\end{remark}

\section{Coefficient Bounds}
\label{sec:coeffs}

The coefficients of stochastic gradient flow and ridge regression may be close, even though the risks are not.  Therefore, here, we pursue bounds on the coefficient error, $\E_{\eta,Z} \| \hat \beta^\sgf(t) - \hat \beta^\ridge(1/t) \|_2^2$.  We start by giving a tight bound on the distance between the coefficents of gradient flow and ridge regression.

%[Coefficient bound for GF vs ridge]
\begin{lemma} \label{thm:coeffs:first_term3}
Fix $X$.  Let $t \geq 0$.  Define
\[
g(t) = 
\begin{cases}
\frac{(1-\exp(-L t)) (1 + L t)}{L t}, & t \leq 1.7933/L \\
\frac{(1-\exp(-\mu t)) (1 + \mu t)}{ \mu t}, & t \geq 1.7933/\mu \\
1.2985, & 1.7933/L < t < 1.7933/\mu \\
\end{cases}.
\]
Then,
\[
\E_\eta \| \hat \beta^\gf(t) - \hat \beta^\ridge(1/t) \|_2^2 \leq (g(t)-1)^2 \cdot \E_\eta \| \hat \beta^\ridge(1/t) \|_2^2.
\]
\end{lemma}

Figure \ref{fig:g} plots the function $g(t)$, defined in the lemma.  We see that $g(t)$ has a maximum of 1.2985, and tends to 1 as either $t \to 0$ or $t \to \infty$.  The behavior makes sense, as both $\hat \beta^\gf(t)$ and $\hat \beta^\ridge(1/t)$ tend to the null model as $t \to 0$, and the min-norm solution as $t \to \infty$.

\begin{figure}[ht]
\vskip -0.2in
\begin{center}
\centerline{\includegraphics[width=0.7\columnwidth]{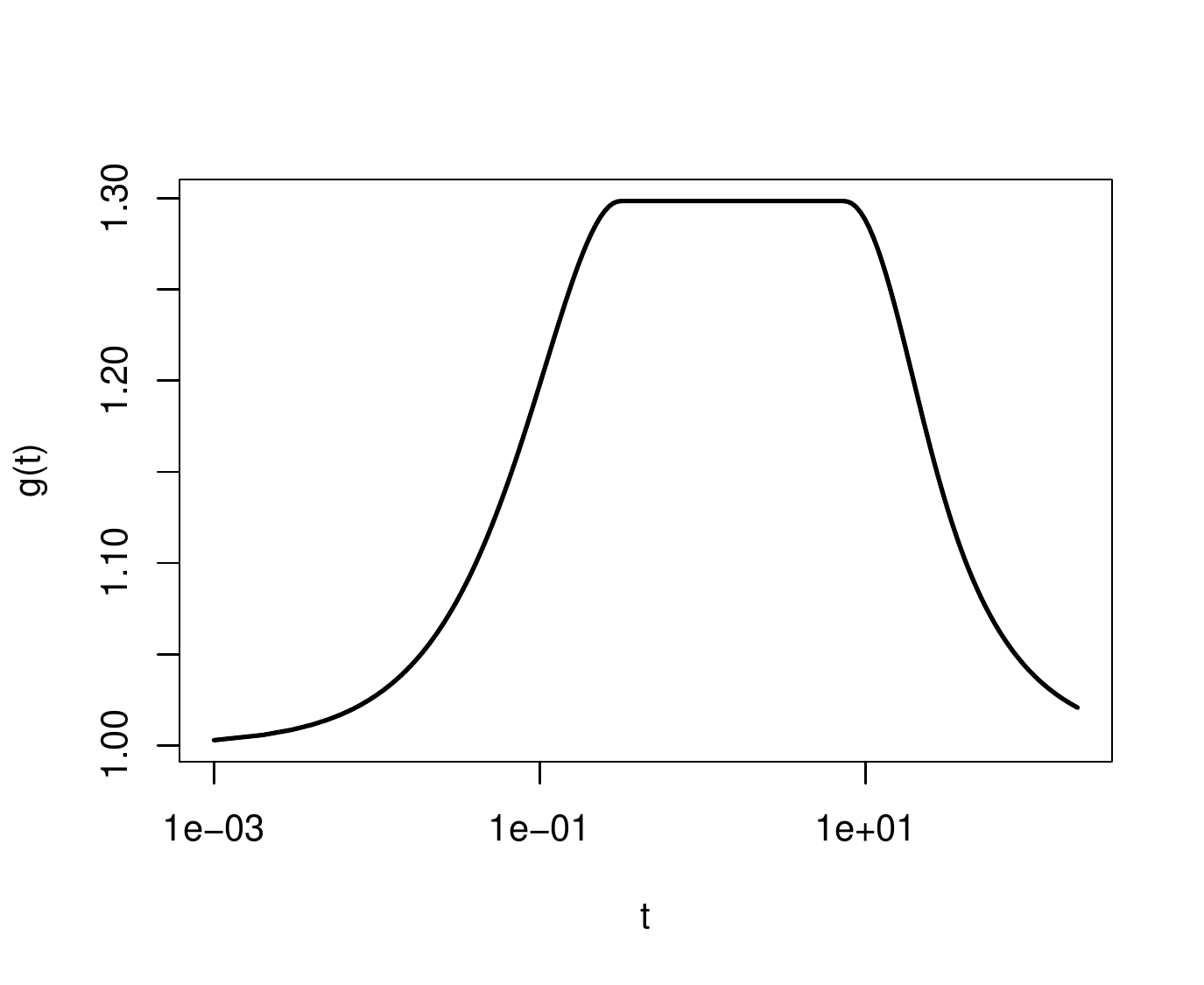}}
\vskip -0.2in
\caption{\textit{The function $g(t)$, defined in Lemma \ref{thm:coeffs:first_term3}.}}
\label{fig:g}
\end{center}
\vskip -0.2in
\end{figure}

Our main result now follows easily, by putting Lemma \ref{thm:coeffs:first_term3} together with Lemma \ref{thm:coeffs:second_term} from Section \ref{sec:risk}.

%[Coefficient bound for SGF vs ridge] 
\begin{theorem}\label{thm:coeffs}
Fix $X$.  Set $\epsilon$ as in Lemma \ref{dyna:exp:loss}.  Let $t > 0$.  Then,
\begin{align*}
& \E_{\eta,Z} \| \hat \beta^\sgf(t) - \hat \beta^\ridge(1/t) \|_2^2 \\
& \leq (g(t)-1)^2 \cdot \E_\eta \| \hat \beta^\ridge(1/t) \|_2^2 + \epsilon \cdot \frac{n}{m} \sum_{i=1}^p \nu_i(t).
\end{align*}
\end{theorem}

\begin{proof}
Expanding $\E_{\eta,Z} \| \hat \beta^\sgf(t) - \hat \beta^\ridge(1/t) \|_2^2$, adding and subtracting $\| \E_{\eta,Z} ( \hat \beta^\sgf(t) ) \|_2^2$, and rearranging yields
\begin{align*}
& \E_{\eta,Z} \| \hat \beta^\sgf(t) - \hat \beta^\ridge(1/t) \|_2^2 \\
& = \E_\eta \| \E_{\eta,Z} ( \hat \beta^\sgf(t) ) - \hat \beta^\ridge(1/t) \|_2^2 + \tr \E_\eta \big[ \Cov_Z (\hat \beta^\sgf(t)) \big].
\end{align*}
As $Q_\epsilon(\beta(t))^{1/2}$ is continuous, it follows from standard properties of Brownian motion (\eg, Theorem 3.2.1 in \citet{oksendal2003stochastic}) that $\E_Z ( \hat \beta^\sgf(t) )  = \hat \beta^\gf(t)$.  Therefore, we have
\begin{align*}
& \E_{\eta,Z} \| \hat \beta^\sgf(t) - \hat \beta^\ridge(\lambda) \|_2^2 \\
& = \E_\eta \| \hat \beta^\gf(t) - \hat \beta^\ridge(1/t) \|_2^2 + \tr \E_\eta \big[ \Cov_Z (\hat \beta^\sgf(t)) \big].
\end{align*}
Lemma \ref{thm:coeffs:first_term3} gives a bound on the first term in the preceding display.  Lemma \ref{thm:coeffs:second_term} and the same arguments used in the proof of Theorem \ref{thm:risk} give a bound on the second term.  Putting the pieces together yields the result.
\end{proof}

%The key component in the proof is again the tight bound on the variance due to mini-batching, given in Lemma \ref{thm:coeffs:second_term}.

\begin{remark}
A bound on the coefficient error, $\E_{\eta,Z} \| \hat \beta^\sgf(t) - \hat \beta^\ridge(1/t) \|_2^2$, is in some sense fundamental, since the risks are close when the coefficients are.  Nonetheless, obtaining risk bounds directly (as was done in Section \ref{sec:risk}) is still interesting, as these can be sharper.
\end{remark}

\begin{remark}
It is possible to give a similar, albeit less sharp, result for any convex loss satisfying a restricted secant inequality \citep{zhang2013gradient}, and noise process satisfying a suitable boundedness condition \citep{vaswani2018fast}.
\end{remark}

\section{Numerical Examples}
\label{sec:exps}

We give numerical examples supporting our theoretical findings.  We generated the data matrix according to $X = \Sigma^{1/2} W$, where the entries of $W$ were i.i.d.~following a normal distribution.  We allow for correlations between the features, setting the diagonal entries of the predictor covariance $\Sigma$ to 1, and the off-diagonals to 0.5.  Below, we present results for $n=100$, $p=500$, and $m=20$.  The supplement gives additional examples with different problem sizes and data models (Student-t and Bernoulli data); the results are similar.  We set $\epsilon =$ 2.2548e-4, following Lemma \ref{dyna:exp:loss}.

Figure \ref{fig:risk} plots the risk of ridge regression, discrete-time SGD \eqref{eq:sgd}, and Theorem \ref{thm:risk}.  For ridge, we used a range of 200 tuning parameters $\lambda$, equally spaced on a log scale from $2^{-15}$ to $2^{15}$.  The expression for the risk of ridge is well-known.  For Theorem \ref{thm:risk}, we set $t = 1/\lambda$.  For SGD, we computed its effective time, using $t = k \epsilon$ and $t = 1/\lambda$.  As for its risk, following the decomposition given in Section \ref{sec:risk}, we first computed the bias and variance of discrete-time gradient descent, using Lemma 3 in \citet{ali2018continuous}, and then added in the variance given by Lemma \ref{mom-match}.  As a comparison, Figure \ref{fig:risk} also plots the risks of gradient flow \eqref{eq:gf:soln}, coming from Lemma 5 in \citet{ali2018continuous}, and discrete-time gradient descent (as was just discussed).
%the bound \eqref{eq:10b} from 

Though the risks look similar, there are subtle differences (the supplement gives examples with larger step sizes and smaller mini-batch sizes, where the differences are more pronounced).  We also see that Theorem \ref{thm:risk} tracks the risk of SGD closely.  In fact, the maximum ratio, across the entire path, of the risk of stochastic gradient flow to that of ridge is 2.5614, whereas the same ratio for SGD to ridge is 1.7214.  Figure \ref{fig:risk} also shows the (optimal) time where each method balances its bias and variance.  Choosing a tuning parameter by balancing bias and variance is common in nonparametric regression, and doing so here implies that stochastic gradient flow stops earlier than gradient flow, because the effective variance terms \eqref{eq:eff} are nonnegative.  We find the optimal stopping times chosen by balancing bias and variance vs.~directly minimizing risk are generally similar.  Moreover, the ratio of the (optimal) risks at these times is 1.0032, indicating that stochastic gradient flow strikes a favorable computational-statistical trade-off.
%maximum
%In our experiments, w

Turning to Theorem \ref{thm:coeffs}, we consider the same experimental setup as before, now plotting the bound of Theorem \ref{thm:coeffs}, and the actual coefficient error $\E_{\eta,Z} \| \hat \beta^\sgf(t) - \hat \beta^\ridge(1/t) \|_2^2$, averaged over 30 draws of $y$ (the underlying coefficients were drawn from a normal distribution, and scaled so the signal-to-noise ratio was roughly 1).  We see the bound tracks the underlying error closely, and is quite small---indicating a tight relationship between stochastic gradient flow and ridge.  For larger $t$, some looseness in the bound is evident, arising from the constants appearing in Lemma \ref{dyna:exp:loss}; giving sharper constants is an important problem for future work.

\begin{figure}[ht]
%\vskip -0.2in
\begin{center}
\centerline{
\includegraphics[width=\columnwidth]{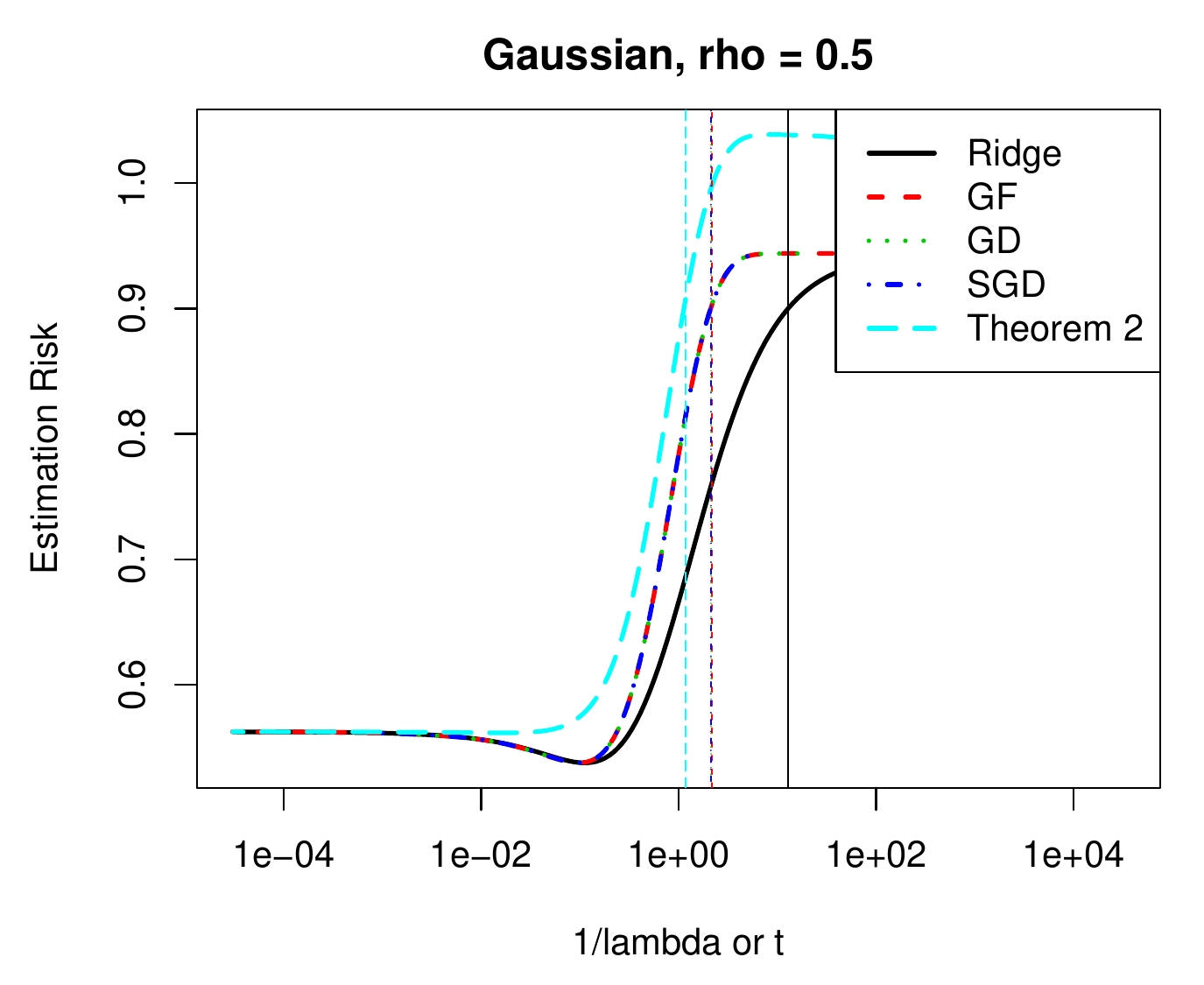}
}
\vskip -0.2in
\caption{\textit{Risks for ridge, SGD, stochastic gradient flow, and gradient descent/flow.  The excess risk of stochastic gradient flow over ridge is the distance between the cyan and black curves.  The vertical lines show the stopping times that balance bias and variance.}}
\label{fig:risk}
\end{center}
\vskip -0.2in
\end{figure}
%regression

\begin{figure}[ht]
\vskip -0.3in
\begin{center}
\centerline{
\includegraphics[width=\columnwidth]{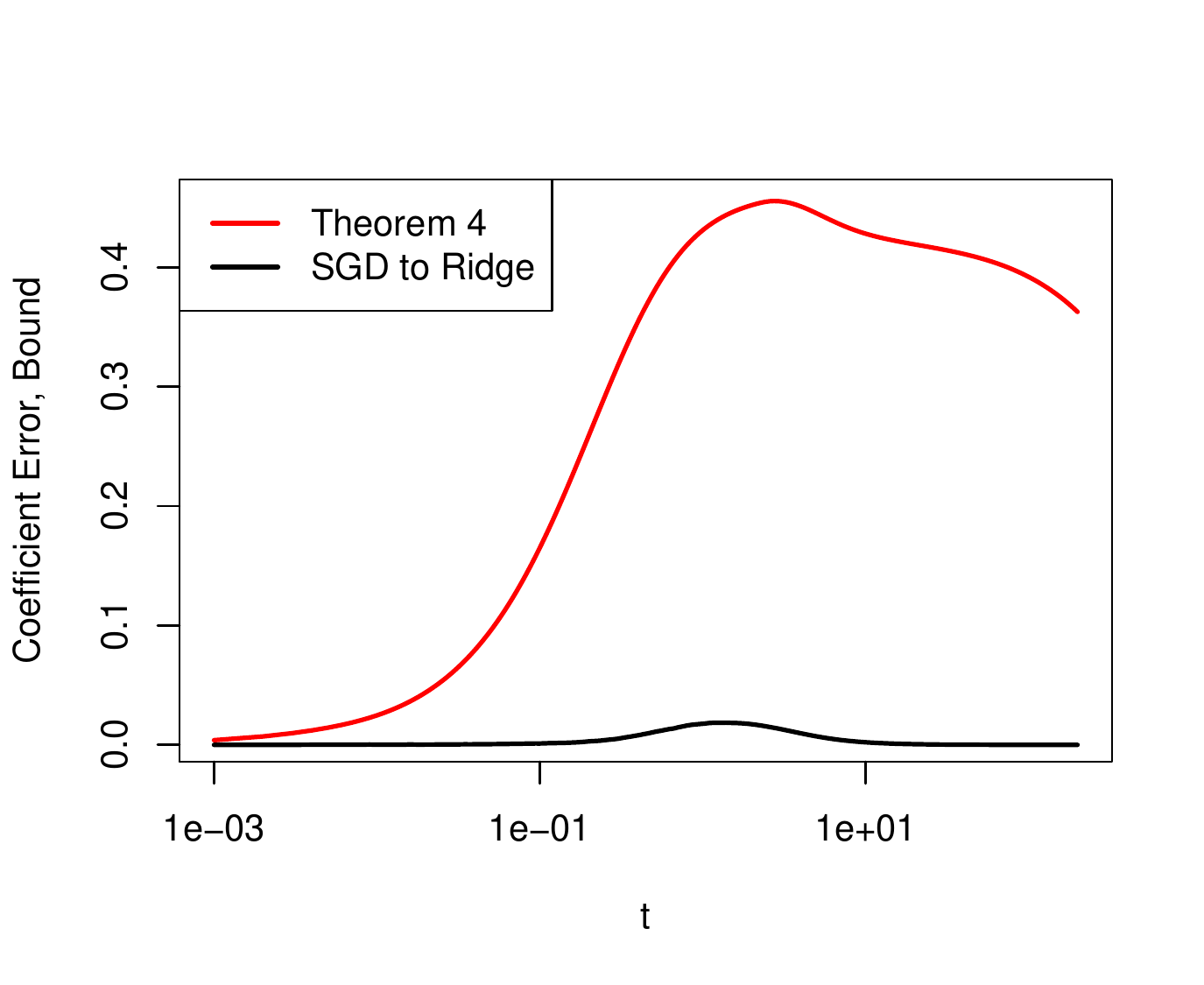}
}
\vskip -0.2in
\caption{\textit{Comparison between Theorem \ref{thm:coeffs}, and the actual distance between the coefficients of SGD and ridge.}}
\label{fig:coeff}
\end{center}
\vskip -0.3in
\end{figure}

\section{Discussion}
\label{sec:disc}

We studied the implicit regularization of stochastic gradient flow, giving theoretical and empirical support for the claim that the method is closely related to $\ell_2$ regularization.  There are a number of important directions for future work, \eg, establishing that stochastic gradient flow and SGD are in fact close, in a precise sense; considering general convex losses; and analyzing adaptive stochastic gradient methods.

\section{Acknowledgements}
We thank a number of people for helpful discussions, including Misha Belkin, Quanquan Gu, J.~Zico Kolter, Jason Lee, Yi-An Ma, Jascha Sohl-Dickstein, Daniel Soudry, and Matus Telgarsky.  ED was supported in part by NSF BIGDATA grant IIS 1837992 and NSF TRIPODS award 1934960.  Part of this work was completed while ED was visiting the Simons Institute.

%\clearpage
\bibliography{alnur}
\bibliographystyle{icml2020}

\appendix
\renewcommand\thesection{S.\arabic{section}}  % Special for supp.
\renewcommand\theequation{S.\arabic{equation}}
\renewcommand\thelemma{S.\arabic{lemma}}
\renewcommand\thetable{S.\arabic{table}}
\renewcommand\thefigure{S.\arabic{figure}}
\renewcommand\thealgorithm{S.\arabic{algorithm}}
\setcounter{figure}{0}
\setcounter{section}{0}
\setcounter{equation}{0}
\setcounter{lemma}{0}
\setcounter{algorithm}{0}
\onecolumn
\section*{Supplementary Material}

\section{Proof of Lemma \ref{mom-match}}
For simplicity, below we will omit the source of the randomness for the various estimators. Implicitly, the randomness is from minibatching in SGD, and from the normal random increments in the discretization of SGD (which we wil can dSGF).

By taking expectations in the SGD iteration, we find
\begin{equation*}
\E\beta^{(k)} = \E\beta^{(k-1)} + \epsilon \cdot \E\frac{1}{n} X^T (y - X \beta^{(k-1)}).
\end{equation*}
This identity only uses that the stochastic gradients are unbiased estimators for the true gradients. Thus, it is true even more generally for any loss function, not just for quadratic loss. However, for quadratic loss, we have a very special property, namely that the \emph{gradient is linear} in the parameter. Using this, we can move the expectation inside, and we find
\begin{equation*}
\E\beta^{(k)} = \E\beta^{(k-1)} + \epsilon \cdot \frac{1}{n} X^T (y - X \E\beta^{(k-1)}).
\end{equation*}
From this, it follows by a direct induction argument that $\E\beta^{(k)}=\beta_{gd}^{(k)}$, where $\beta_{gd}^{(k)}$ is the gradient descent iteration with learning rate $\ep$ started from 0. Indeed, for $k=0$, we have $\E\beta^{(k)}=\beta_{gd}^{(k)}=0$. Next, the two sequences satisfy the same recurrence. Hence the induction finishes the argument.

A similar reasoning holds for dSGF. This shows that $\E_{\tilde Z} \tilde \beta^{(k)} = \E_{\mathcal I_1,\ldots,\mathcal I_k} \beta^{(k)}$.

By taking the covariance of the SGD iteration, conditionally on the previous iterate $\beta^{(k-1)}$, we find
\begin{align*}
\Cov[\beta^{(k)}|\beta^{(k-1)}] & =
\epsilon^2 \cdot \Cov \Bigg[ \frac{1}{m} X_{\mathcal I_k}^T (y_{\mathcal I_k} - X_{\mathcal I_k} \beta^{(k-1)}) - \frac{1}{n} X^T (y - X \beta^{(k-1)}) \Bigg]\\
 &=
\epsilon \cdot  Q_\epsilon(\beta^{(k-1)}).
\end{align*}
Note that the definition of $Q$ already includes an $\ep$. This shows that the conditional covariance of the SGD iterate, conditoned on the previous iterate, is the same as for dSGF. As before, this observation holds not just for quadratic objectives, but also for general objectives. However, noting that $Q$ is a \emph{quadratic function} of the parameter, it follows from an inductive argument that the covariance matrix of the iterates $\beta^{(k)}$ and $\tilde \beta^{(k)}$ equals at every iteration. Indeed, the reason is that the covariance at each iteration only depends on second order statistics of the previous iteration (including the mean and the covariance), and so the induction step will hold. This shows that $\Cov_{\tilde Z} \tilde \beta^{(k)} = \Cov_{\mathcal I_1,\ldots,\mathcal I_k} \beta^{(k)}$, finishing the proof.

We can also find the explicit form of the recursion. While this is not required in the statement of the lemma, it is used in our numerical examples.

\begin{align*}
\Cov[\beta^{(k)}]  &= \E\Cov[\beta^{(k)}|\beta^{(k-1)}] & \\
&=
\frac{\ep^2}{mn} \sum_{i=1}^n \E (y_i-x_i^T \beta^{(k-1)})^2 x_ix_i^T
-
\frac{\ep^2}{mn^2} 
\E(\sum_{i=1}^n \E (y_i-x_i^T \beta^{(k-1)})x_i)^{\otimes 2}
\end{align*}
For the first term, we can write
\begin{align*}
&=
\frac{\ep^2}{mn} \sum_{i=1}^n
\E (y_i-x_i^T \E\beta^{(k-1)}+x_i^T \E\beta^{(k-1)}-x_i^T \beta^{(k-1)})^2 x_ix_i^T\\
&=
\frac{\ep^2}{mn} \sum_{i=1}^n
(y_i-x_i^T \beta_{gd}^{(k-1)})^2x_ix_i^T
+
\frac{\ep^2}{mn} \sum_{i=1}^n
[x_i^T (\E\beta^{(k-1)}- \beta^{(k-1)})]^2  x_ix_i^T\\
&=
\frac{\ep^2}{mn} \sum_{i=1}^n
(y_i-x_i^T \beta_{gd}^{(k-1)})^2x_ix_i^T
+
\frac{\ep^2}{mn} \sum_{i=1}^n
x_i^T \Cov[\beta^{(k-1)}]  x_i\cdot  x_ix_i^T.
\end{align*}
For the first term, we can write

\begin{align*}
\sum_{i=1}^n \E(y_i-x_i^T \beta^{(k-1)})x_i
&=
\sum_{i=1}^n  (y_i-x_i^T \E\beta^{(k-1)}) x_i\\
&=
\sum_{i=1}^n  (y_i-x_i^T \beta_{gd}^{(k-1)}) x_i
+
\sum_{i=1}^n  x_i^T [\beta_{gd}^{(k-1)}-\E\beta^{(k-1)}] x_i\\
&=
\sum_{i=1}^n  (y_i-x_i^T \beta_{gd}^{(k-1)}) x_i
+
\sum_{i=1}^n  n \hSigma [\beta_{gd}^{(k-1)}-\E\beta^{(k-1)}]
\end{align*}
so
\begin{align*}
&\frac{\ep^2}{mn^2} 
\E(\sum_{i=1}^n \E (y_i-x_i^T \beta^{(k-1)})x_i)^{\otimes 2}\\
&=
\frac{\ep^2}{mn^2} 
[\sum_{i=1}^n  (y_i-x_i^T \beta_{gd}^{(k-1)}) x_i]^{\otimes 2}
+
\frac{\ep^2}{m} 
\E[\hSigma [\beta_{gd}^{(k-1)}-\E\beta^{(k-1)}]]^{\otimes 2}\\
&=
\frac{\ep^2}{mn^2} 
[\sum_{i=1}^n  (y_i-x_i^T \beta_{gd}^{(k-1)}) x_i]^{\otimes 2}
+
\frac{\ep^2}{m} 
\hSigma \Cov[\beta^{(k-1)}] \hSigma.
\end{align*}

This gives an explicit linear recursion for the covariance matrices. The first term can be viewed as a covariance matrix of the gradients evaluated at the mean value of the process (i.e., at the value of the GD iteration). The second term depends on the covariance of the previous iteration.

%This is a direct consequence of the previous results.

\section{Proof of Lemma \ref{lem:soln}}
As the diffusion coefficient $Q_\epsilon(\beta(t))^{1/2}$ is Lipschitz continuous and positive semidefinite, standard results from numerical analysis (\eg, \citet{oksendal2003stochastic}) show that the solution to the differential equation \eqref{eq:sgf:nonconstcov} exists and is unique.

Now consider the process $\tilde \beta(t) = \exp ( t  \hSigma ) \beta(t)$.  By Ito's lemma,
\[
d \tilde \beta(t) = \hSigma \exp ( t  \hSigma ) \beta(t) dt + \exp ( t  \hSigma ) d \beta(t).
\]
Plugging in the expression for $d \beta(t)$ from \eqref{eq:sgf:nonconstcov} and simplifying, we see that
\[
d \tilde \beta(t) = \exp ( t  \hSigma ) \Big( \frac{1}{n} X^T y \Big) dt + \exp ( t  \hSigma ) Q_\epsilon(\beta(t))^{1/2} d W(t),
\]
or, equivalently,
\[
\tilde \beta(t) = \int_0^t \exp ( \tau  \hSigma ) \Big( \frac{1}{n} X^T y \Big) d\tau + \int_0^t \exp \Big( \tau  \hSigma \Big) Q_\epsilon(\beta(\tau))^{1/2} d W(\tau).
\]
Changing variables back yields
\begin{align*}
\beta(t) = \exp ( - t  \hSigma ) \cdot \int_0^t \exp ( \tau  \hSigma ) \Big( \frac{1}{n} X^T y \Big) d\tau + \exp ( - t  \hSigma ) \cdot \int_0^t \exp ( \tau  \hSigma ) Q_\epsilon(\beta(\tau))^{1/2} d W(\tau).
\end{align*} 

Considering only the first integral above, by arguments similar to those given in Lemma 1 of \citet{ali2018continuous}, we obtain
\[
\int_0^t \exp ( \tau  \hSigma ) \Big( \frac{1}{n} X^T y \Big) d\tau = \big( \exp ( t  \hSigma ) - I \big) (X^T X)^+ X^T y,
\]
and so
\begin{align*}
& \exp (-t  \hSigma ) \cdot \int_0^t \exp ( \tau  \hSigma ) \Big( \frac{1}{n} X^T y \Big) d\tau
= (X^T X)^+ \big( I - \exp ( -t  \hSigma ) \big) X^T y
 = \hat \beta^\gf(t),
\end{align*}
which gives the result.  \hfill\qed

\section{Proof of Lemma \ref{lem:nonconst_vs_const}}
To keep things simple, we prove the result using the uncentered covariance matrix of the stochastic gradients, \ie, we let $Q_\epsilon(\hat \beta^\sgf(t)) = (1/(n m)) X^T F(\hat \beta^\sgf(t)) X$, where $h(\beta) = ( y_1 - x_1^T \beta, \ldots, y_n - x_n^T \beta )$ are the residuals at $\beta$, and $F(\beta) = \diag( h(\beta) )^2$.  A similar result holds for the actual covariance matrix, but it is a little difficult to interpret.

Calculations similar to those given in Lemma \ref{thm:coeffs:second_term} (appearing below) show
\begin{align*}
\E_Z \| \hat \beta^\sgf(t) - \tilde \beta^\sgf(t) \|_2^2 & = \E \int_0^t \tr \Bigg[ \exp ( (\tau- t)  \hSigma )
\Bigg( Q_\epsilon(\hat \beta^\sgf(\tau))^{1/2} - \Big( \frac{\epsilon}{m} \cdot \hSigma \Big)^{1/2} \Bigg)^2 \exp ( (\tau- t)  \hSigma ) \Bigg] d \tau.
\end{align*}

Continuing on, and writing $L = \lambda_{\max}(\hat \Sigma)$, we have
\begin{align*}
\E_Z \| \hat \beta^\sgf(t) - \tilde \beta^\sgf(t) \|_2^2 & = \E_Z \int_0^t \tr \Bigg[ \Bigg( Q_\epsilon(\hat \beta^\sgf(\tau))^{1/2} - \Big( \frac{\epsilon}{m} \cdot \hSigma \Big)^{1/2} \Bigg)^2 \exp(2 (\tau - t) \hat \Sigma) \Bigg] d \tau \\
& = \E_Z \int_0^t \tr \Bigg[ V^T \Bigg( Q_\epsilon(\hat \beta^\sgf(\tau))^{1/2} - \Big( \frac{\epsilon}{m} \cdot \hSigma \Big)^{1/2} \Bigg)^2 V \exp(2 (\tau - t) S) \Bigg] d \tau \\
& \leq \E_Z \int_0^t \tr \Bigg[ V^T \Bigg( Q_\epsilon(\hat \beta^\sgf(\tau))^{1/2} - \Big( \frac{\epsilon}{m} \cdot \hSigma \Big)^{1/2} \Bigg)^2 V \Bigg] \tr \big[ \exp(2 (\tau - t) S) \big] d \tau \\
& \leq \frac{4 L p^2 \epsilon}{m} \cdot \int_0^t \sum_{i=1}^n \E_Z \Big[ \big| (y_i - x_i^T \hat \beta^\sgf(\tau))^2 - 1 \big| \Big] \tr \big[ \exp(2 (\tau - t) \hat \Sigma) \big] d \tau \\
& \leq \frac{4 L p^3 \epsilon}{m} \cdot \int_0^t \sum_{i=1}^n \E_Z \Big[ \big| (y_i - x_i^T \hat \beta^\sgf(\tau))^2 - 1 \big| \Big] d \tau,
\end{align*}
where we used the eigendecomposition $\hat \Sigma = V S V^T$ on the second line, the helper Lemma \ref{thm:err:helper} (appearing below) on the third, the helper Lemma \ref{thm:err:helper2} (appearing below) on the fourth, and the fact that the map $A \mapsto \tr \exp(A)$ is operator monotone on the fifth.  This proves the result.  \hfill\qed

\begin{lemma} \label{thm:err:helper}
Let $A \in \R^{n \times n}$ be a nonnegative diagonal matrix, and $B \in \R^{n \times n}$ be a positive semidefinite matrix.  Then $\tr(AB) \leq \tr(A) \tr(B)$.
\end{lemma}

\begin{proof}
Write $\tr(AB) = \sum_{i=1}^n A_{ii} B_{ii}$.  Cauchy-Schwarz shows that
\begin{align*}
\sum_{i=1}^n A_{ii} B_{ii} \leq \Big( \sum_{i=1}^n A_{ii}^2 \Big)^{1/2} \Big( \sum_{i=1}^n B_{ii}^2 \Big)^{1/2}.
\end{align*}
Using the simple fact that $\| x \|_2 \leq \| x \|_1$, along with the fact that $A,B$ have nonnegative diagonal entries, now yields the result.
\end{proof}

\begin{lemma} \label{thm:err:helper2}
Fix $y$, $X$ and $\beta$.  Let $h(\beta) = ( y_1 - x_1^T \beta, \ldots, y_n - x_n^T \beta )$ denote the residuals at $\beta$, and $F(\beta) = \diag( h(\beta) )^2$.  Then,
\begin{align*}
\tr \Bigg[ \Bigg( Q_\epsilon(\beta)^{1/2} - \Big( \frac{\epsilon}{m} \cdot \hSigma \Big)^{1/2} \Bigg)^2 \Bigg] \leq  \frac{4 L p^2\epsilon}{m}  \cdot \tr \big[ | F(\beta) - I | \big],
\end{align*}
where the absolute value is to be interpreted elementwise.
\end{lemma}

\begin{proof}
Using the matrix perturbation inequality given in Lemma A.2 of \citet{nguyen2019bridging}, we see that
\begin{align}
\tr \Bigg[ \Bigg( Q_\epsilon(\beta)^{1/2} - \Big( \frac{\epsilon}{m} \cdot \hSigma \Big)^{1/2} \Bigg)^2 \Bigg] & = \Bigg\| Q_\epsilon(\beta)^{1/2} - \Big( \frac{\epsilon}{m} \cdot \hSigma \Big)^{1/2} \Bigg\|_F^2 \notag \\
& \leq p \cdot \Bigg\| Q_\epsilon(\beta)^{1/2} - \Big( \frac{\epsilon}{m} \cdot \hSigma \Big)^{1/2} \Bigg\|_2^2 \notag \\
& \leq 4 p^2 \cdot \Big\| Q_\epsilon(\beta) - \frac{\epsilon}{m} \cdot \hSigma \Big\|_2. \label{eq:20}
\end{align}

Noting the expression for the covariance matrix of the stochastic gradients given in \eqref{eq:Q:3}, we obtain for \eqref{eq:20} that
\begin{align*}
4 p^2 \cdot \Big\| Q_\epsilon(\beta) - \frac{\epsilon}{m} \cdot \hSigma \Big\|_2 & = \frac{4 p^2 \epsilon}{m} \cdot \Bigg\| \frac{1}{n} X^T \big( F(\beta) - I \big) X \Bigg\|_2 \\
& \leq \frac{4 L p^2 \epsilon}{m} \cdot \| F(\beta) - I \|_2 \\
& \leq \frac{4 L p^2 \epsilon}{m} \cdot \tr \big[ | F(\beta) - I | \big].
\end{align*}
Here, we let $L = \lambda_{\max}(\hat \Sigma)$, $h(\beta) = ( y_1 - x_1^T \beta, \ldots, y_n - x_n^T \beta )$ denote the residuals at $\beta$, and $F(\beta) = \diag( h(\beta) )^2$.  This shows the result.
%The final inequality above followed by putting together the simple facts that (i) $\| X \|_2 \leq \| X \|_F$, (ii) $F(\beta)$ is diagonal, and (iii) $\| x \|_2 \leq \| x \|_1$.  
\end{proof}

\section{Proof of Lemma \ref{thm:coeffs:second_term}}
As $\hat \beta^\gf(t)$ is constant and the Brownian motion term in \eqref{eq:sgf:nonconstcov} has mean zero, we have
\begin{align*}
& \tr \Cov_Z (\hat \beta^\sgf(t)) = \tr \E_Z \Bigg[ \exp ( - t  \hSigma ) \Bigg( \int_0^t \exp ( \tau  \hSigma ) Q_\epsilon(\hat \beta^\sgf(\tau))^{1/2} d W(\tau) \Bigg) \\
& \hspace{3.5in} \Bigg( \int_0^t \exp ( \tau  \hSigma ) Q_\epsilon(\hat \beta^\sgf(\tau))^{1/2} d W(\tau) \Bigg)^T \exp ( - t  \hSigma ) \Bigg].
\end{align*}
Using Ito's isometry along with the linearity of the trace, we obtain
\begin{align}
& \tr \Cov_Z (\hat \beta^\sgf(t))
 = \E_Z \int_0^t \tr \Bigg[
\exp ( (\tau- t) \hSigma ) Q_\epsilon(\hat \beta^\sgf(\tau))
\exp ( (\tau- t) \hSigma ) \Bigg] d \tau. \label{eq:1}
\end{align}

For the squared error loss, the covariance matrix of the stochastic gradients at $\beta$ sampled with replacement has a relatively well-known simplified form (\cf~\citet{hoffer2017train,zhang2017determinantal,hu2017diffusion}).  Let $h(\beta) = ( y_1 - x_1^T \beta, \ldots, y_n - x_n^T \beta )$ denote the residuals at $\beta$, $F(\beta) = \diag( h(\beta) )^2$, and $\tilde F(\beta) = n^{-1} h(\beta) h(\beta)^T$.  Then,
\begin{align}
Q_\epsilon(\beta) & = \Cov_{\mathcal I} \Bigg( \frac{1}{m} X_{\mathcal I}^T(y_{\mathcal I} - X_{\mathcal I} \beta) \Bigg) \notag \\
& = \frac{1}{nm} X^T (F(\beta) - \tilde F(\beta)) X \notag \\ %\label{eq:Q:2}
& \preceq \frac{1}{nm} X^T F(\beta) X. \label{eq:Q:3}
\end{align}

Letting $A = \exp( (\tau- t) \hSigma)$, the trace appearing in \eqref{eq:1} may be expressed as
\[
\frac{\epsilon}{mn} \cdot \tr \big( A  X^T F(\hat \beta^\sgf(\tau)) X  A \big) = \frac{\epsilon}{mn} \cdot \tr \big( F(\hat \beta^\sgf(\tau)) X  A^2  X^T \big).
\]
Since $X  A^2  X^T$ is positive semidefinite, the matrix $F(\hat \beta^\sgf(\tau)) X  A^2  X^T$ is the product of a nonnegative diagonal matrix and a positive semidefinite matrix; this satisfies the conditions of Lemma \ref{thm:err:helper}, which yields the bound
\begin{align*}
 \tr \big( F(\hat \beta^\sgf(\tau)) X  A^2  X^T \big)
& \leq \tr \big( F(\hat \beta^\sgf(\tau)) \big) \tr ( X  A^2  X^T ).
\end{align*}
By straightforward manipulations, we see $(1/n) \tr ( X  A^2  X^T ) = \tr \big( \hat \Sigma \exp(2 (\tau - t) \hat \Sigma) \big)$.  Therefore, $\tr \Cov_Z (\hat \beta^\sgf(t))$, as in \eqref{eq:1}, may be bounded as
\begin{align*}
\tr \Cov_Z(\hat \beta^\sgf(t)) & \leq \frac{\epsilon}{m} \cdot \E_Z \int_0^t \tr \big( F(\hat \beta^\sgf(\tau)) \big) \tr \big( \hat \Sigma \exp(2 (\tau - t) \hat \Sigma) \big) d \tau \notag \\
& = \frac{\epsilon}{m} \cdot \int_0^t \E_Z \big[ \tr \big( F(\hat \beta^\sgf(\tau)) \big) \big] \tr \big( \hat \Sigma \exp(2 (\tau - t) \hat \Sigma) \big) d \tau. \notag %\label{eq:19}
\end{align*}
The equality followed by Fubini's theorem (which applies here, since the product of the trace of a nonnegative diagonal matrix and the trace of a positive semidefinite matrix, is nonnegative).  As $\E_Z \big[ \tr \big( F(\hat \beta^\sgf(\tau)) \big) \big] = 2n f(\hat \beta^\sgf(\tau))$, this shows the result.  \hfill\qed

\section{Proof of Lemma \ref{dyna:exp:loss}}
In this lemma, it will be helpful to start slightly more generally, with the SDE for SGF on a general loss function $g$. The specific proofs of this lemma are in Sections \ref{op} and \ref{up}.

To approximate discrete time SGD with learning rate $\ep$ and batch size $m$, it is not hard to see that the same logic we have used throughout the paper leads to the SDE  
$$d\bt=-\nabla g(\bt)dt + \eta \sigma(\bt)dW(t)$$
where $\sigma(\bt)\sigma(\bt)^T$ is the covariance of the gradients at parameter value $\bt$, and $\eta = \ep/\sqrt{m}$. 

We derive the SDE for the behavior of the loss function itself, for a general loss. For gradient flow on a loss function $g$, i.e., the dynamics $d\bt=-\nabla g(\bt)dt$, it is well known that the dynamics induced on the loss function is:
\begin{align*}
dg(\bt) =-|\nabla g(\bt)|^2dt.
\end{align*}
This shows that the loss function is always non-increasing, i.e., that gradient flow is a descent method. In contrast, we will find that the loss for stochastic gradient flow is \emph{not} always non-increasing. We mention that a related calculation has been performed in \cite{zhu2018anisotropic}, under different assumptions (starting from a local min, integrating over time), and for a different purpose (to understand dynamics escaping local minima).

\begin{proposition}[Dynamics of the loss for SGF] For SGF on a loss function $g$, the value of the loss function $g$ evolves according to the following SDE:
\begin{align*}
dg(\bt)
& =\left(-|\nabla g(\bt)|^2 + \frac{\eta^2}2 \tr[ \nabla^2 g(\bt) \cdot \sigma(\bt)\sigma(\bt)^T]\right)dt
- \eta \nabla g(\bt)^T \sigma(\bt) dW(t),
\end{align*}
where $\Sigma(\beta) := \sigma(\beta)\sigma(\beta)^T$ is the covariance of the stochastic gradients at parameter value $\beta$. Also $\eta = \ep/\sqrt{m}$, where SGF approximates discrete time SGD with learning rate $\ep$ and batch size $m$. This can be written in a distributionally equivalent way as
\begin{align*}
dg(\bt)
& =\left(-|\nabla g(\bt)|^2 + \frac{\eta^2}2 \tr[\Sigma(\bt) \cdot \nabla^2 g(\bt)]\right)dt
- \eta \sqrt{\nabla g(\bt)^T \Sigma(\bt) \nabla g(\bt)}dZ(t),
\end{align*}
where $Z=\{Z(t)\}_{t\ge 0}$ is a 1-dimensional Brownian motion.

For the special case of least squares, let $r(t) = Y- X\beta(t)$ be the residual and define
$M(t) = [XX^T \diag[r(t)^2] XX^T]/n^2.$ Then we find that the equation for SGF with second moment matrix of stochastic gradients is 
\begin{align*}
dg(\bt)
& =\left(-\frac{\|X^T r(t)\|^2}{n^2}+ \frac{\ep^2\cdot \tr [M(t)]}{m}\right)dt
+\frac{\ep\cdot\|M(t)^{1/2} r(t)\|}{m^{1/2} n^{1/2}} dZ(t).
%\label{grsde}
\end{align*}
where $Z=\{Z(t)\}_{t\ge 0}$ is a 1-dimensional Brownian motion.
\end{proposition}

\begin{proof}
We start with the SDE for SGF
$$d\bt=-\nabla g(\bt)dt + \eta \sigma(\bt)dW(t)$$
where $\sigma(\bt)\sigma(\bt)^T$ is the covariance of the gradients at parameter value $\bt$. Then, Ito's rule leads to
\begin{align*}
dg(\bt) &=\nabla g(\bt)^T d\bt + \frac12 d\bt^T \cdot H_xg\cdot d\bt \\
& =\left(-|\nabla g(\bt)|^2 + \frac{\eta^2}2 \tr[\sigma(\bt)^T \cdot \nabla^2 g(\bt) \cdot \sigma(\bt)]\right)dt
- \eta \nabla g(\bt)^T \sigma(\bt) dW(t).
\end{align*}

For the special case of least squares loss, we have the following. We have already calculated most terms, and we have in addition that the second moment matrix of the gradients is
$$\sigma(\bt)\sigma(\bt)^T = \frac{1}{mn}  X^T \diag[r(t)^2] X$$
where $r(t) = Y- X\beta(t)$ is the residual. Plugging in the terms for least squares,
\begin{align*}
dg(\bt)
& =\left(-\frac{\|X^T r(t)\|^2}{n^2}+ \frac{\ep^2\cdot \tr [XX^T \diag[r(t)^2] XX^T]}{mn^2}\right)dt
+\frac{\ep\cdot\|\diag[r(t)] XX^T r(t)\|}{m^{1/2} n^{3/2}} dZ(t).
\end{align*}
Here $Z=\{Z(t)\}_{t\ge 0}$ is a 1-dimensional Brownian motion, which is obtained by transforming the original diffusion term, which is a linear combination of the entries of $dW(t)$, into a distributionally equivalent 1-dimensional process.
Letting
$$M(t) = \frac{XX^T \diag[r(t)^2] XX^T}{n^2}$$
we can simplify the above as
\begin{align*}
dg(\bt)
& =\left(-\frac{\|X^T r(t)\|^2}{n^2}+ \frac{\ep^2\cdot \tr [M(t)]}{m}\right)dt
+\frac{\ep\cdot\|M(t)^{1/2} r(t)\|}{m^{1/2} n^{1/2}} dZ(t).
\end{align*}
\end{proof}
Comparing this with the noiseless case, i.e., when $\eta=0$, we note that \emph{both} the drift and the diffusion terms have changed. The drift term is reduced by a term that is proportional to $\eta^2$. The diffusion term is new altogether. This shows that for sufficiently large $\eta$, the drift will not be positive, and hence the process will not converge to a point mass limit distribution.

Let us start with studying the diffusion with the second moment matrix first. We will show a geometric contraction of the loss. We can bound the terms in the drift term as follows.
We have (using $\odot$ for elementwise product of two conformable vectors or matrices)
\begin{align*}
r^2 \odot vec[\diag[(XX^T)^2]]&\le \|r\|^2 \max_i [(X^T X)^2]_{ii}\\
\|X^T r(t)\|^2 &\ge \sigma_{\min}(X^T)^2 \|r(t)\|^2
\end{align*}
The second inequality holds with $\sigma_{\min}(X^T)$ being the smallest nonzero singular value of $X^T$.  Why?  Because it is easy to see that we always have the decomposition
\begin{equation}
\frac{1}{2n} \| y - X \beta \|_2^2 = \frac{1}{2n} \| P_{\col(X)} y - X \beta \|_2^2 + \frac{1}{2n} \| P_{\nul(X^T)} y \|_2^2, \label{eq:13}
\end{equation}
\ie, we may think of the residual $r(t)$ above (and in the remainder) as $P_{\col(X)}(y - X \beta)$.  It follows that $\| X^T r(t) \|_2^2 \geq s \cdot \| r(t) \|_2^2$, where $s$ denotes the smallest nonzero singular value of $X^T$.  It is also clear that the expressions for the gradient flow and stochastic gradient flow solutions do not change, if we use the decomposition in \eqref{eq:13} as the loss.  Hence, in the remainder, we always write $\sigma_{\min}$ to mean the smallest nonzero singular value.  Below, we consider the situation when $p > n$ separately from the case $n \geq p$.
%Note that $X^T$ is a $p\times n$ matrix, and so if $p\ge n$, and $X^T$ is full rank, then $\sigma_{\min}(X^T)>0$. Let us study this case first.  (We may study this case without a loss of generality: the same proof goes through without placing any restrictions on the rank of $X$, by leveraging the fact that $\beta(0) = 0$ and that $\beta(t)$ stays in the row space of $X$.)
%The second inequality holds with $\sigma_{\min}(X^T)$ being the smallest singular value of $X^T$. Note that $X^T$ is a $p\times n$ matrix, and so if $p\ge n$, and $X^T$ is full rank, then $\sigma_{\min}(X^T)>0$. Let us study this case first.  (We may study this case without a loss of generality: the same proof goes through without placing any restrictions on the rank of $X$, by leveraging the fact that $\beta(0) = 0$ and that $\beta(t)$ stays in the row space of $X$.)

\subsection{Overparametrized case}
\label{op}

We find, with $M(t) = [XX^T \diag[r(t)^2] XX^T]/n^2,$
\begin{align*}
\left(-\frac{\|X^T r(t)\|^2}{n^2}+ \frac{\ep^2\cdot \tr [M(t)]}{m}\right)&\le - \alpha \|r(t)\|^2\\
\alpha= \sigma_{\min}(X)^2/n^2&-\ep^2 \max_i [(X^T X)^2]_{ii}/m.
\end{align*}
Then by taking expectations in the SDE for the loss, we find $l'(t)\le - \alpha l(t)$.
Hence $l(t)\le \exp(-\alpha t) l_0$. For the diffusion with the covariance matrix of the gradients as a diffusion term, $\Sigma(\bt) \prec \E \bt\bt^T$, hence $\tr[\Sigma(\bt) \cdot \nabla^2 g(\bt)] \le \tr[\E \bt\bt^T \cdot \nabla^2 g(\bt)]$. Thus, the drift term in this case is at most as large as the one in the second moment case, and so the contraction happens at least as fast. This proves the claim for the covariance matrix of the gradients as a diffusion term. The same argument will also apply to this case when $p<n$.

\subsection{Underparametrized case}
\label{up}

Now, if $p<n$, then in this case, in general the loss cannot converge to zero, because the number of equations is larger than the number of constraints. Instead, the loss converges close to the loss of the OLS estimator:
%Now, if $p<n$, then we have $\sigma_{\min}(X^T)=0$. And indeed, in this case in general the loss cannot converge to zero, because the number of equations is larger than the number of constraints. Instead, the loss converges close to the loss of the OLS estimator:
$$l^* = \|Y-X\hbeta^{ols}\|^2/(2n) = \|P_X^\perp Y\|^2/(2n).$$
Then we can write
$$l(t)-l^*=\|X(\beta(t)-\hbeta^{ols})\|^2/(2n)$$
Moreover, $X^T r(t)= - X^T X(\beta(t)-\hbeta^{ols})$. Also, letting $b=P_X^\perp Y$, $r(t) = b+X(\hbeta^{ols}-\beta(t))$, and hence

$$\|r(t)\|^2 \le 2(\|b\|^2+\|X(\hbeta^{ols}-\beta(t))\|^2)$$
so that
\begin{align*}
r^2 \odot \diag[(XX^T)^2]&\le \|r\|^2 \max_i [(X^T X)^2]_{ii} \le v + 2\|X(\hbeta^{ols}-\beta(t))\|^2 \max_i [(X^T X)^2]_{ii}
\end{align*}
where $v= 2 \|b\|^2 \max_i [(X^T X)^2]_{ii}$. Let $q(t)=X(\beta(t)-\hbeta^{ols})$.
Then
\begin{align*}
\left(-\frac{\|X^T r(t)\|^2}{n^2}+ \frac{\ep^2\cdot \tr [M(t)]}{m}\right)&\le - c_0 \|q(t)\|^2+C_0\\
c_0= \sigma_{\min}(X)^2/n^2&- 2\ep^2\max_i [(X^T X)^2]_{ii}/m\\
C_0=  2\ep^2\|b\|^2& \max_i [(X^T X)^2]_{ii}/m
\end{align*}
By taking expectations in the SDE for the loss, we find
$$l'(t)\le - c_0[l(t)-l^*]+C_0.$$
Or also
$$l'(t)\le - c_0[l(t)-(l^*+C_0/c_0)].$$

This shows that $l$ converges geometrically to the level $l^*+C_0/c_0$, which is higher than the minimum OLS loss. In this case, the additional fluctuations occur because of the inherent noise in the algorithm.

\section{Calculations for the In-Sample Risks, for Theorem \ref{thm:risk}}
For in-sample risk, we have the bias-variance decomposition
\begin{align*}
\Risk(\hat \beta^\sgf(t); \beta_0) & = \| \E_{\eta, Z} (\hat \beta^\sgf(t)) - \beta_0 \|_{\hat \Sigma}^2 + \tr \big[ \Cov_{\eta,Z} (\hat \beta^\sgf(t)) \hat \Sigma \big] \\
& = \| \E_{\eta, Z} (\hat \beta^\sgf(t)) - \beta_0 \|_{\hat \Sigma}^2 + \tr \big[ \Cov_\eta( \hat \beta^\gf(t) ) \hat \Sigma \big] + \E_\eta \tr \big[ \Cov_Z( \hat \beta^\sgf(t) \,|\, \eta ) \hat \Sigma \big],
\end{align*}
where we write $\| x \|_A^2 = x^T A x$.

Following the same logic as in the proof of Theorem \ref{thm:risk}, we see that
\begin{align*}
\| \E_{\eta, Z} (\hat \beta^\sgf(t)) - \beta_0 \|_{\hat \Sigma}^2 & = \Big( \Bias^\inn(\hat \beta^\sgf(t); \beta_0) \Big)^2 = \Big( \Bias^\inn(\hat \beta^\gf(t); \beta_0) \Big)^2 \\
\tr \big[ \Cov_\eta( \hat \beta^\gf(t) ) \hat \Sigma \big] & = \Var^\inn ( \hat \beta^\gf(t) ) \\
\E_\eta \tr \big[ \Cov_Z( \hat \beta^\sgf(t) \,|\, \eta ) \hat \Sigma \big] & \leq \epsilon \cdot \frac{n}{m} \sum_{i=1}^p \Big( \frac{\tilde w s_i}{s_i - u/2} \big( \exp(-u t) - \exp(-2 t s_i) \big) + \tilde v \big( 1 - \exp(-2 t s_i) \big) \Big) s_i,
\end{align*}
where (\cf~Lemma 5 in \citet{ali2018continuous})
\begin{align*}
\Big( \Bias^\inn(\hat \beta^\gf(t); \beta_0) \Big)^2 & = \sum_{i=1}^p (v_i^T \beta_0)^2 s_i \exp(-2 t s_i) \\
\Var^\inn(\hat \beta^\gf(t)) & = \frac{\sigma^2}{n} \sum_{i=1}^{p} (1 - \exp(-t s_i)).^2
\end{align*}

\section{Proof of Lemma \ref{thm:relative_risk}}
Looking back at \eqref{aa}, we have
\begin{align}
& \Risk(\hat \beta^\sgf(t); \beta_0) \label{eq:11} \\
& \leq \Bias^2(\hat \beta^\gf(t); \beta_0) + \Var_\eta(\hat \beta^\gf(t)) + \epsilon \cdot \frac{n}{m} \sum_{i=1}^p \Big( \frac{\tilde w s_i}{s_i - u/2} \big( \exp(-u t) - \exp(-2 t s_i) \big) + \tilde v \big( 1 - \exp(-2 t s_i) \big) \Big) \notag \\
& = \Bias^2(\hat \beta^\gf(t); \beta_0) + \Var_\eta(\hat \beta^\gf(t)) + T + \epsilon \cdot \frac{n}{m} \sum_{i=1}^p \tilde v \big( 1 - \exp(-2 t s_i) \big), \notag
\end{align}
where we let
\[
T = \epsilon \cdot \frac{n}{m} \sum_{i=1}^p \Big( \frac{\tilde w s_i}{s_i - u/2} \big( \exp(-u t) - \exp(-2 t s_i) \big) \Big).
\]
Focusing on just $T$ for now, and noting that $s_i > u/2$ for $i=1,\ldots,p$, we see
\[
T \leq p \tilde w \epsilon \cdot \frac{n}{m} \frac{\mu}{\mu - u/2} \exp(-u t),
\]
implying that
\[
(T / \alpha)^{2 L / u} \leq \exp(-2 L t) \leq \exp(-2 s_i t),
\]
for $i=1,\ldots,p$, which means that
\[
\| \beta_0 \|_2^2 \cdot (T / \alpha)^{2 L / u} = \| V \beta_0 \|_2^2 \cdot (T / \alpha)^{2 L / u} \leq \sum_{i=1}^p (v_i^T \beta_0)^2 \exp(-2 s_i t) = \Bias^2(\hat \beta^\gf(t); \beta_0).
\]
Here, we used the eigendecomposition $\hat \Sigma = V S V^T$, and Lemma 5 in \citet{ali2018continuous}.  Hence,
\begin{align}
T & = \epsilon \cdot \frac{n}{m} \sum_{i=1}^p \Big( \frac{\tilde w s_i}{s_i - u/2} \big( \exp(-u t) - \exp(-2 t s_i) \big) \Big) \notag \\
& \leq \alpha \cdot \Bigg( \frac{ | \Bias(\hat \beta^\gf(t); \beta_0) | }{ \| \beta_0 \|_2 } \Bigg)^{\mu/L} \notag \\
& = \delta \cdot | \Bias(\hat \beta^\gf(t); \beta_0) |^{1/\kappa}. \label{eq:12}
\end{align}
%the mean value theorem applied to the map $x \mapsto \exp(-x)$ shows, for $z \in (u t, \, 2 s_i t)$,
%\begin{align*}
%\frac{\tilde w s_i}{s_i - u/2} \big( \exp(-u t) - \exp(-2 t s_i) \big) & = \frac{\tilde w s_i}{s_i - u/2} (2 t s_i - u t) \exp(-z) \\
%& \leq \frac{(2 t s_i - u t) \exp(-u t) \tilde w s_i}{s_i - u/2} \\
%& = \frac{2 t (2 s_i - u) \exp(-u t) \tilde w s_i}{2 s_i - u} \\
%& = 2 \tilde w t s_i \exp(-u t).
%\end{align*}
%Hence,
%\begin{equation}
%\epsilon \cdot \frac{n}{m} \sum_{i=1}^p \Big( \frac{\tilde w s_i}{s_i - u/2} \big( \exp(-u t) - \exp(-2 t s_i) \big) \Big) \leq \epsilon \cdot \frac{n}{m} \sum_{i=1}^p 2 \tilde w t s_i \exp(-u t). \label{eq:12}
%\end{equation}

Therefore, putting \eqref{eq:11} and \eqref{eq:12} together, along with Lemma 5 in \citet{ali2018continuous}, we obtain
\begin{align*}
& \Risk(\hat \beta^\sgf(t); \beta_0) \\
& \leq \Bias^2(\hat \beta^\gf(t); \beta_0) + \delta \cdot | \Bias(\hat \beta^\gf(t); \beta_0) |^{1/\kappa} + \underbrace{\frac{\sigma^2}{n} \sum_{i : s_i > 0} \frac{(1-\exp(-t s_i))^2}{s_i}}_{A} + \underbrace{\tilde v \epsilon \cdot \frac{n}{m} \sum_{i : s_i > 0} (1 - \exp(-2 t s_i))}_{B}.
\end{align*}
%\begin{align*}
%& \Risk(\hat \beta^\sgf(t); \beta_0) \\
%& \leq \Bias^2(\hat \beta^\gf(t); \beta_0) + \underbrace{\frac{\sigma^2}{n} \sum_{i : s_i > 0} \frac{(1-\exp(-t s_i))^2}{s_i}}_{A} + \underbrace{\tilde v \epsilon \cdot \frac{n}{m} \sum_{i : s_i > 0} (1 - \exp(-2 t s_i))}_{B} + \delta t \exp(-u t).
%\end{align*}

Now for convenience, write $A = \frac{\sigma^2}{n} \sum_{i : s_i > 0} a_i$ and $B = \frac{\epsilon}{2m} \cdot \sum_{i : s_i > 0} b_i$.  Let $f(x)$ be the continuous extension of $\tilde f(x)$, where $\tilde f(x) = x \frac{1+\exp(-x)}{1-\exp(-x)}$ (\ie, $f(x) = \tilde f(x)$ when $x > 0$, but $f(x) = 2$ when $x = 0$).  It can be checked that $f(x)$ is nondecreasing, so that $\sup_{x \in [0,L]} f(tx) = f(tL)$.

As $f(t s_i) \leq f(tL)$, we have for each $i$ such that $s_i > 0$,
\[
t s_i \frac{1+\exp(-t s_i)}{1-\exp(-t s_i)} \leq f(tL).
\]
Multiplying both sides by $(1-\exp(-t s_i))^2$ and rearranging yields
\[
(1+\exp(-t s_i)) (1-\exp(-t s_i)) \leq \frac{f(tL)}{t} \frac{(1-\exp(-t s_i))^2}{s_i},
\]
\ie,
\[
b_i \leq \frac{f(tL)}{t} a_i.
\]
Therefore,
\[
A + B \leq \Bigg( 1 + \frac{\tilde v \epsilon \cdot n^2 f(tL)}{m t \sigma^2} \Bigg) A.
\]
Now, note that $f(x)$ is increasing on $x > 0$, and $f(x)/x$ is decreasing on $x > 0$.  Also, note that $f(x) \leq 2.164$ when $x \leq 1$, and $f(x)/x \leq 2.164$ when $x > 1$.  Thus, $f(x)/x \leq \max \{ 2.164/x, 2.164 \}$.  So,
\[
A + B \leq \Bigg( 1 + 2.164 \epsilon \cdot \frac{\tilde v n^2 \max ( 1/t, L )}{m \sigma^2} \Bigg) A.
\]

Putting together the pieces, we obtain
\[
\Risk(\hat \beta^\sgf(t); \beta_0) \leq \Bias^2(\hat \beta^\gf(t); \beta_0) + \delta \cdot | \Bias(\hat \beta^\gf(t); \beta_0) |^{1/\kappa} + \gamma(t) \cdot \Var(\hat \beta^\gf(t)),
\]
%\[
%\Risk(\hat \beta^\sgf(t); \beta_0) \leq \Bias^2(\hat \beta^\gf(t); \beta_0) + \gamma(t) \cdot \Var(\hat \beta^\gf(t)) + \delta t \exp(-u t),
%\]
which shows the claim for gradient flow.  Applying Theorem \ref{thm:gf} shows the result for ridge.  Turning to in-sample risk, the exact same bounds actually follow by similar arguments, just as discussed before.  \hfill\qed

%Using part (a) of Theorem 1 in \citet{ali2018continuous}, we have
%\begin{align*}
%& \Risk(\hat \beta^\sgf(t); \beta_0)
% \leq 1.6862 \cdot \Bias^2(\hat \beta^\ridge_{1/t}; \beta_0) + 1.6862 \gamma(t) \cdot \Var(\hat \beta^\ridge_{1/t}) + \delta t \exp(-u t),
%\end{align*}
%which shows the claim for ridge.

%Turning to in-sample risk, the exact same bounds follow by similar arguments, just now starting from \eqref{eq:10b} and redefining $\delta = 2 \epsilon \cdot \tilde w n \tr(\hat \Sigma^2) / m$.  \hfill\qed

\section{Calculations for Remark \ref{rem:const}}
Using Lemma \ref{lem:soln}, we may denote the solution to the time-homogeneous process \eqref{eq:sgf:constcov} as $\tilde \beta^\sgf(t)$.  Then, following the same logic as in the proof of Theorem \ref{thm:risk}, we obtain
\[
\Risk(\tilde \beta^\sgf(t); \beta_0) = \Bias^2(\hat \beta^\gf(t); \beta_0) + \Var_\eta(\hat \beta^\gf(t)) + \tr \E_\eta \big[ \Cov_Z(\tilde \beta^\sgf(t) \,|\, \eta) \big].
\]
Ito's isometry shows that
\begin{align}
\tr \E_\eta \big[ \Cov_Z(\tilde \beta^\sgf(t) \,|\, \eta) \big] & = \frac{\epsilon}{m} \cdot \tr \E_\eta \Bigg[ \Bigg( \exp(-t \hat \Sigma) \int_0^t \exp(\tau \hat \Sigma) \hat \Sigma^{1/2} dW(\tau) \Bigg) \Bigg( \exp(-t \hat \Sigma) \int_0^t \exp(\tau \hat \Sigma) \hat \Sigma^{1/2} dW(\tau) \Bigg)^T \Bigg] \notag \\
& = \frac{\epsilon}{m} \cdot \tr \E_\eta \Bigg[ \Bigg( \int_0^t \exp(-t \hat \Sigma) \exp(\tau \hat \Sigma) \hat \Sigma \exp(\tau \hat \Sigma) \exp(-t \hat \Sigma) d \tau \Bigg] \notag \\
& = \frac{\epsilon}{m} \cdot \tr \Bigg( \hat \Sigma \int_0^t \exp(2 (\tau - t) \hat \Sigma) d \tau \Bigg) \label{eq:8} \\
& = \frac{\epsilon}{2m} \cdot \tr \Big( \hat \Sigma \hat \Sigma^+ (I - \exp(-2 t \hat \Sigma)) \Big). \notag
\end{align}
Finally, expanding $\exp(-2 t \hat \Sigma)$ into its power series representation and using the eigendecomposition $\hat \Sigma = V S V^T$ shows that $\hat \Sigma \hat \Sigma^+ (I - \exp(-2 t \hat \Sigma)) = I - \exp(-2 t \hat \Sigma)$, which gives
\begin{equation}
\Risk(\tilde \beta^\sgf(t); \beta_0) = \Risk(\hat \beta^\gf(t); \beta_0) + \frac{\epsilon}{2m} \cdot \tr \big( I - \exp(-2 t \hat \Sigma) \big). \label{eq:14}
\end{equation}
From \eqref{eq:14}, it is straightforward to derive expressions analogous to those appearing in \eqref{eq:10}, \eqref{eq:10b}, \eqref{eq:excess}, for the process $\tilde \beta^\sgf(t)$.  It is also possible to follow the same logic as in the proof of Lemma \ref{thm:relative_risk} to arrive at a similar expression for $\tilde \beta^\sgf(t)$ (\ie, with $\gamma(t)$ suitably redefined).

Comparing the preceding calculations with those given in the proof of Theorem \ref{thm:risk}, we see that a key simplification occurs in \eqref{eq:8}, above.  Here, the (relatively) complicated expression appearing in \eqref{eq:4c}, 
\[
\frac{2 n \epsilon}{m} \cdot \int_0^t \E_\eta [ f(\hat \beta^\sgf(\tau)) ] d \tau,
\]
is replaced with the comparatively simpler expression $(\epsilon / m) \cdot \hat \Sigma$ in \eqref{eq:8}.  This simplification allows the risk expression in \eqref{eq:14} to hold with equality, though it is evidently less refined than the bound appearing in, \eg, \eqref{eq:10}.

\section{Proof of Lemma \ref{thm:coeffs:first_term3}}
Letting $X = n^{1/2} U S^{1/2} V^T$ be a singular value decomposition, we may express
\begin{align*}
\hat \beta^\gf(t) = n^{-1/2} V S^+ (I - \exp(-t S)) S^{1/2} U^T y = n^{-1/2} t^{1/2} V (t S)^+ (I - \exp(-t S)) (t S)^{1/2} U^T y,
\end{align*}
and
\begin{align*}
\hat \beta^\ridge(1/t) = n^{-1/2} V \Big( S + \frac{1}{t} \cdot I \Big)^{-1} S^{1/2} U^T y = n^{-1/2} t^{1/2} V ( t S + I )^{-1} (t S)^{1/2} U^T y.
\end{align*}
Therefore,
\begin{equation}
\| \hat \beta^\gf(t) - \hat \beta^\ridge(1/t) \|_2 = \Big\| n^{-1/2} t^{1/2} \Big( (t S)^+ (I - \exp(-t S)) - ( t S + I )^{-1} \Big) (t S)^{1/2} U^T y \Big\|_2. \label{eq:23}
\end{equation}

Now define $f(x) = (1-\exp(-x))(1+x)/x$ with domain $x \geq 0$ (let $f(0) = 0$).  Lemma 7 in \citet{ali2018continuous} shows that $f$ attains its unique maximum at $x^* = 1.7933$, where $f(x^*) = 1.2985$.  Moreover, it can be checked that $f$ is unimodal.  This means that, for $i=1,\ldots,p$ and $t \leq 1.7933/L$,
\[
\frac{(1-\exp(-s_i t)) (1 + s_i t)}{s_i t} \leq \frac{(1-\exp(-L t)) (1 + L t)}{L t},
\]
\ie,
\[
\frac{1-\exp(-s_i t)}{s_i t} - \frac{1}{1 + s_i t} \leq (g(t) - 1) \cdot \frac{1}{1 + s_i t}.
\]
Similar reasoning shows that, for $t \geq 1.7933/\mu$,
\[
\frac{1-\exp(-s_i t)}{s_i t} - \frac{1}{1 + s_i t} \leq (g(t) - 1) \cdot \frac{1}{1 + s_i t}.
\]
When $1.7933/L < t < 1.7933/\mu$, we may simply take
\[
\frac{1-\exp(-s_i t)}{s_i t} - \frac{1}{1 + s_i t} \leq (1.2985 - 1) \cdot \frac{1}{1 + s_i t} = (g(t) - 1) \cdot \frac{1}{1 + s_i t}.
\]

Returning to \eqref{eq:23}, we have shown that
\[
(t S)^+ (I - \exp(-t S)) - ( t S + I )^{-1} \preceq (g(t) - 1) \cdot ( t S + I )^{-1}.
\]
Thus,
\begin{align*}
\| \hat \beta^\gf(t) - \hat \beta^\ridge(1/t) \|_2 & \leq \| (g(t) - 1) \cdot n^{-1/2} t^{1/2} ( t S + I )^{-1} (t S)^{1/2} U^T y \|_2 \\
& = (g(t) - 1) \cdot \| \hat \beta^\ridge(1/t) \|_2,
\end{align*}
as claimed.  \hfill\qed

%\clearpage
\section{Additional Numerical Simulations}

\begin{figure*}[h!]
%\vskip -0.2in
\begin{center}
\centerline{
\includegraphics[width=0.49\textwidth]{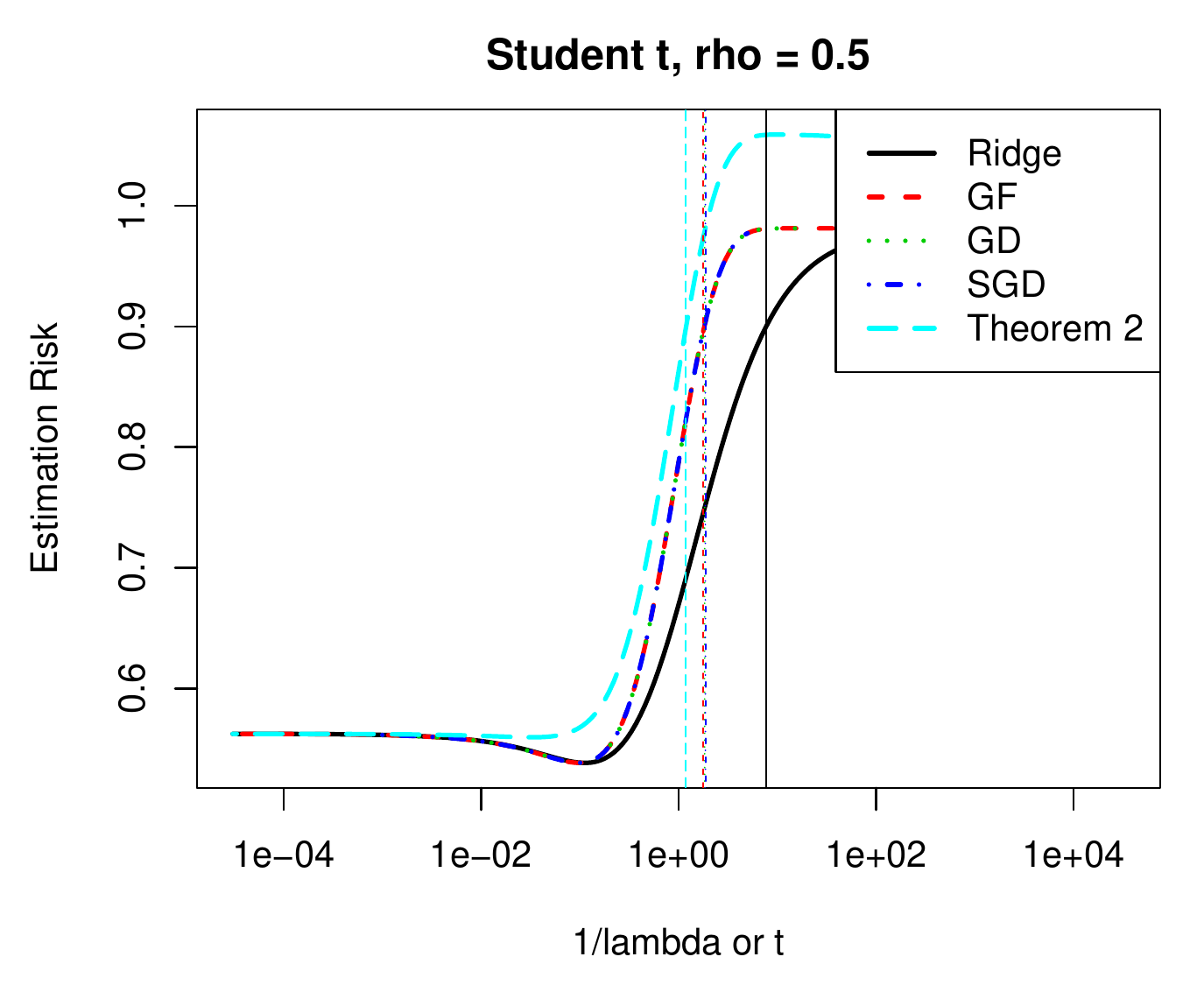} \hfill
\includegraphics[width=0.49\textwidth]{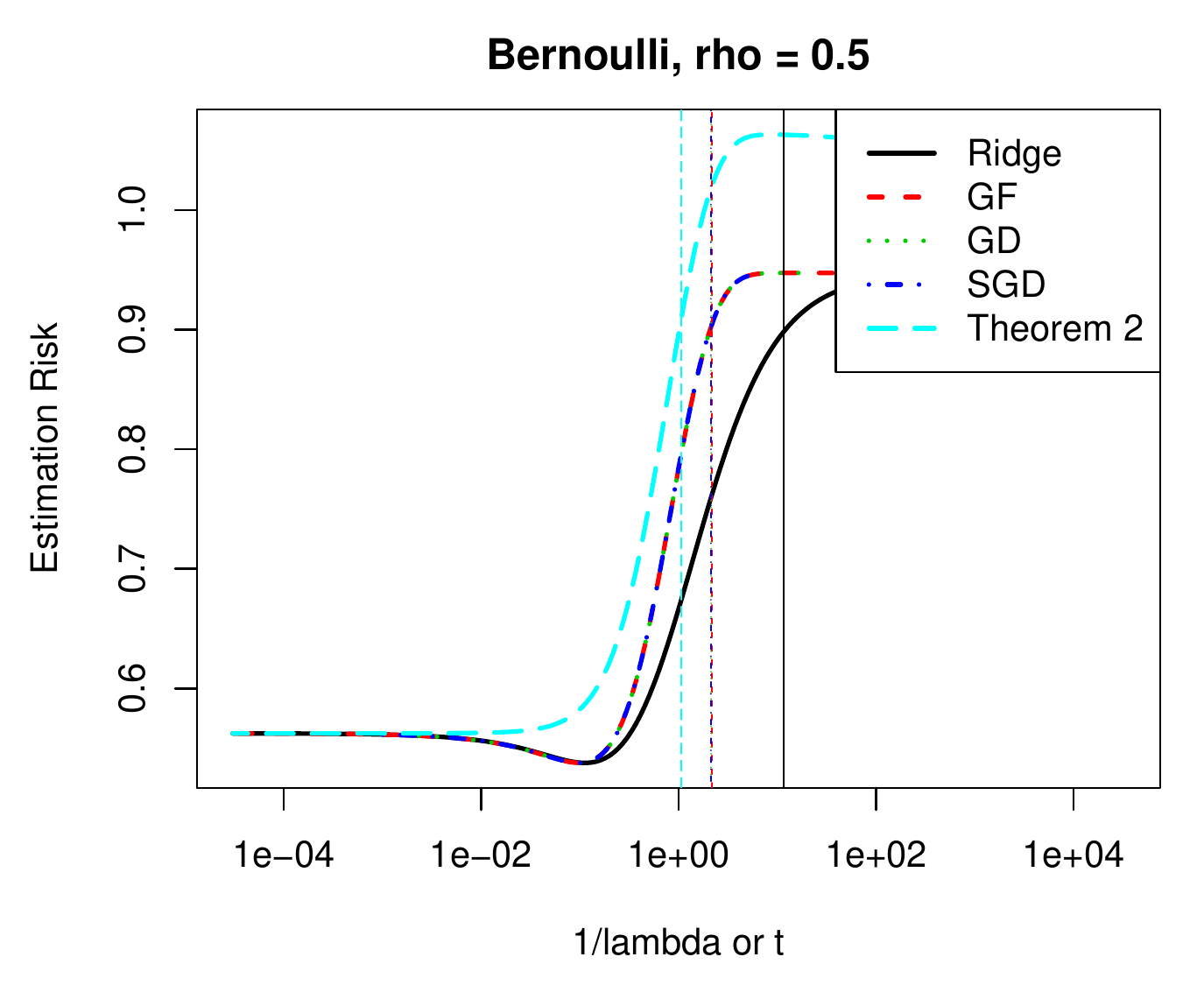}
}
%\vskip -0.4in
\caption{\textit{Risks for ridge regression, discrete-time SGD, stochastic gradient flow (as in Theorem \ref{thm:risk}), discrete-time gradient descent, and gradient flow on Student-t and Bernoulli data, where $n=100$, $p=500$, $m=20$, and $\epsilon$ was set following Lemma \ref{dyna:exp:loss}.  The excess risk of stochastic gradient flow over ridge is given by the distance between the cyan and black curves.  The vertical lines show the stopping times that balance bias and variance.}}
\label{fig:risk:low}
\end{center}
%\vskip -0.2in
\end{figure*}

\begin{figure*}[h!]
%\vskip -0.2in
\begin{center}
\centerline{
\includegraphics[width=0.33\textwidth]{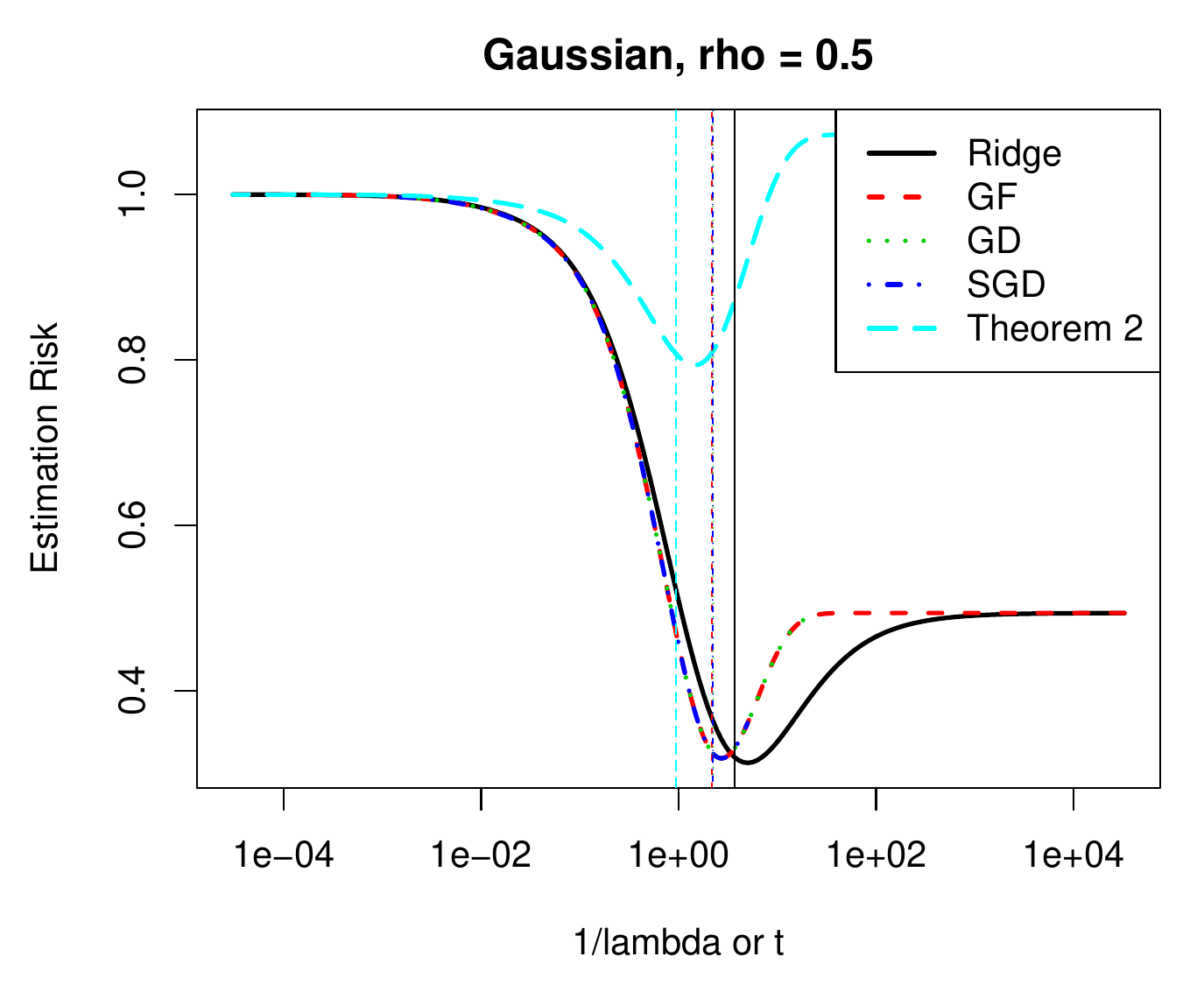} \hfill
\includegraphics[width=0.33\textwidth]{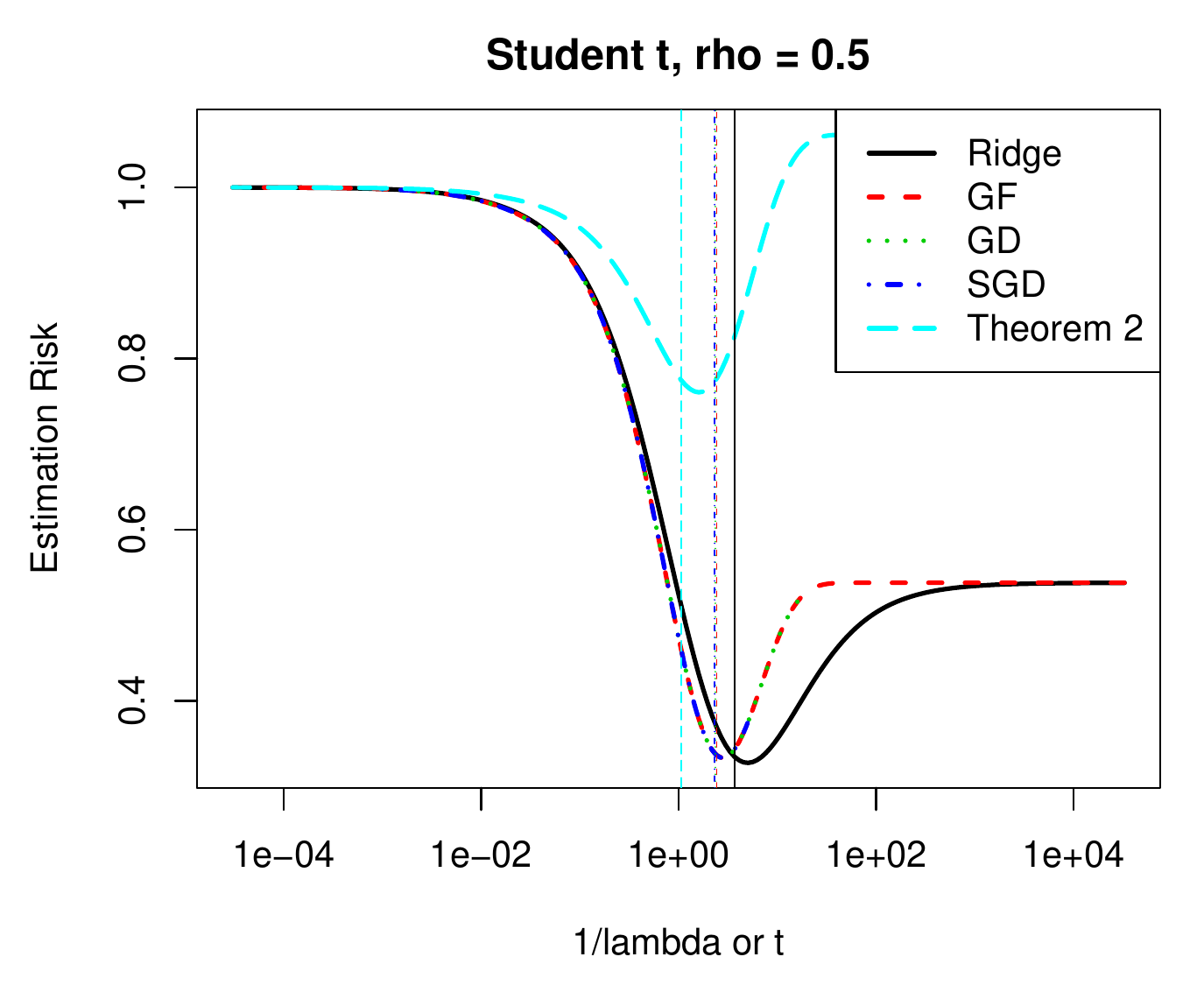} \hfill
\includegraphics[width=0.33\textwidth]{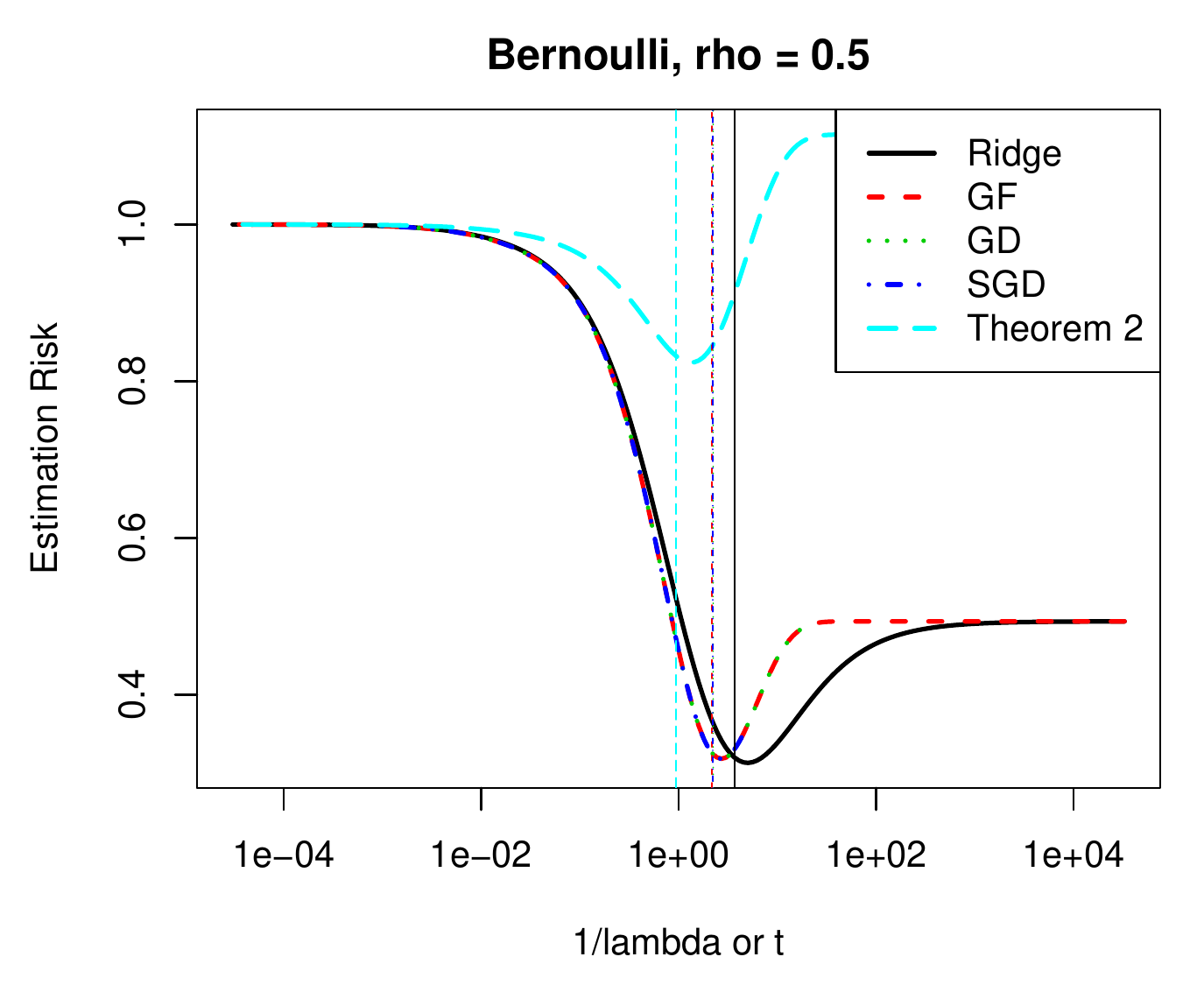}
}
\centerline{
\includegraphics[width=0.33\textwidth]{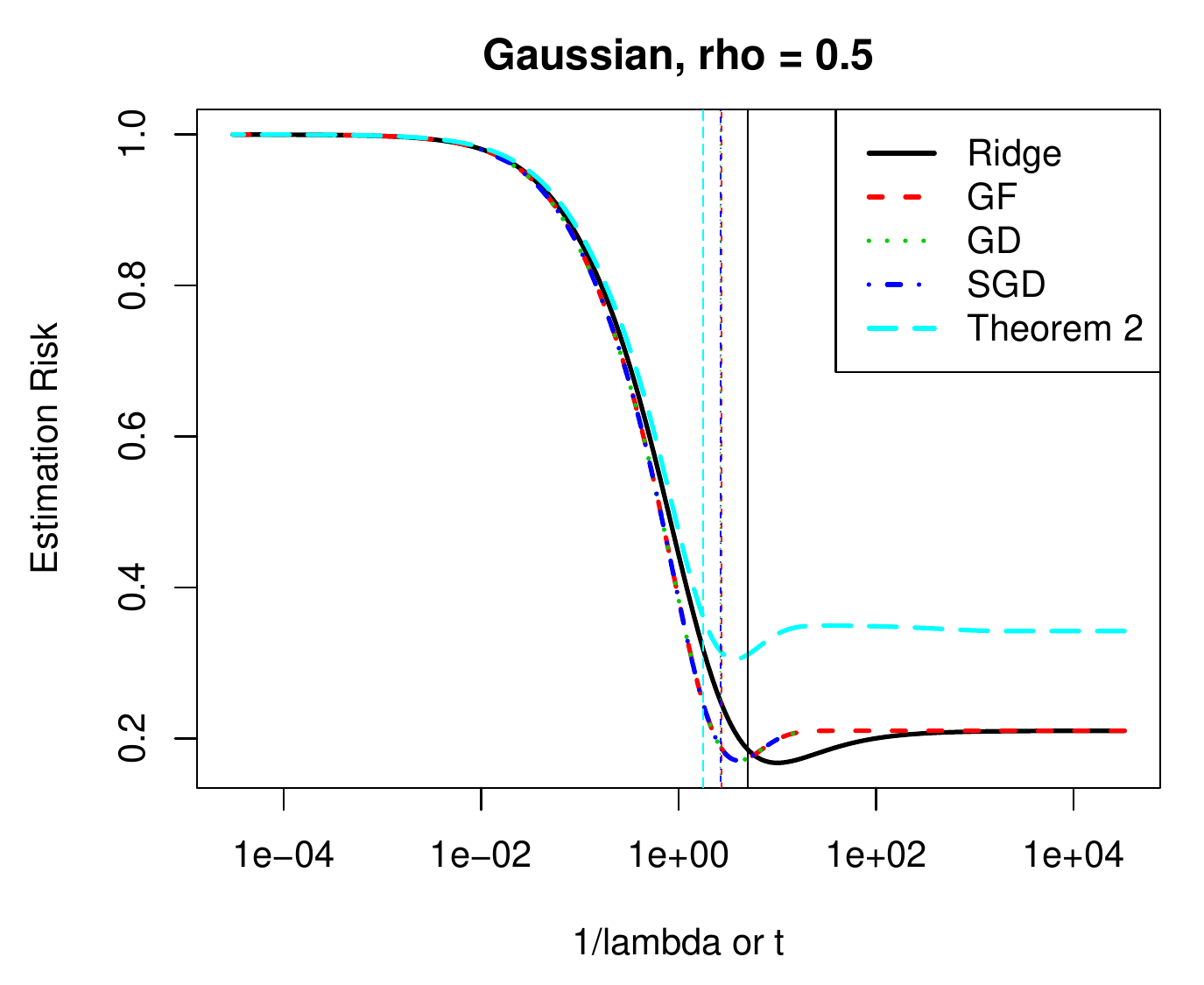} \hfill
\includegraphics[width=0.333\textwidth]{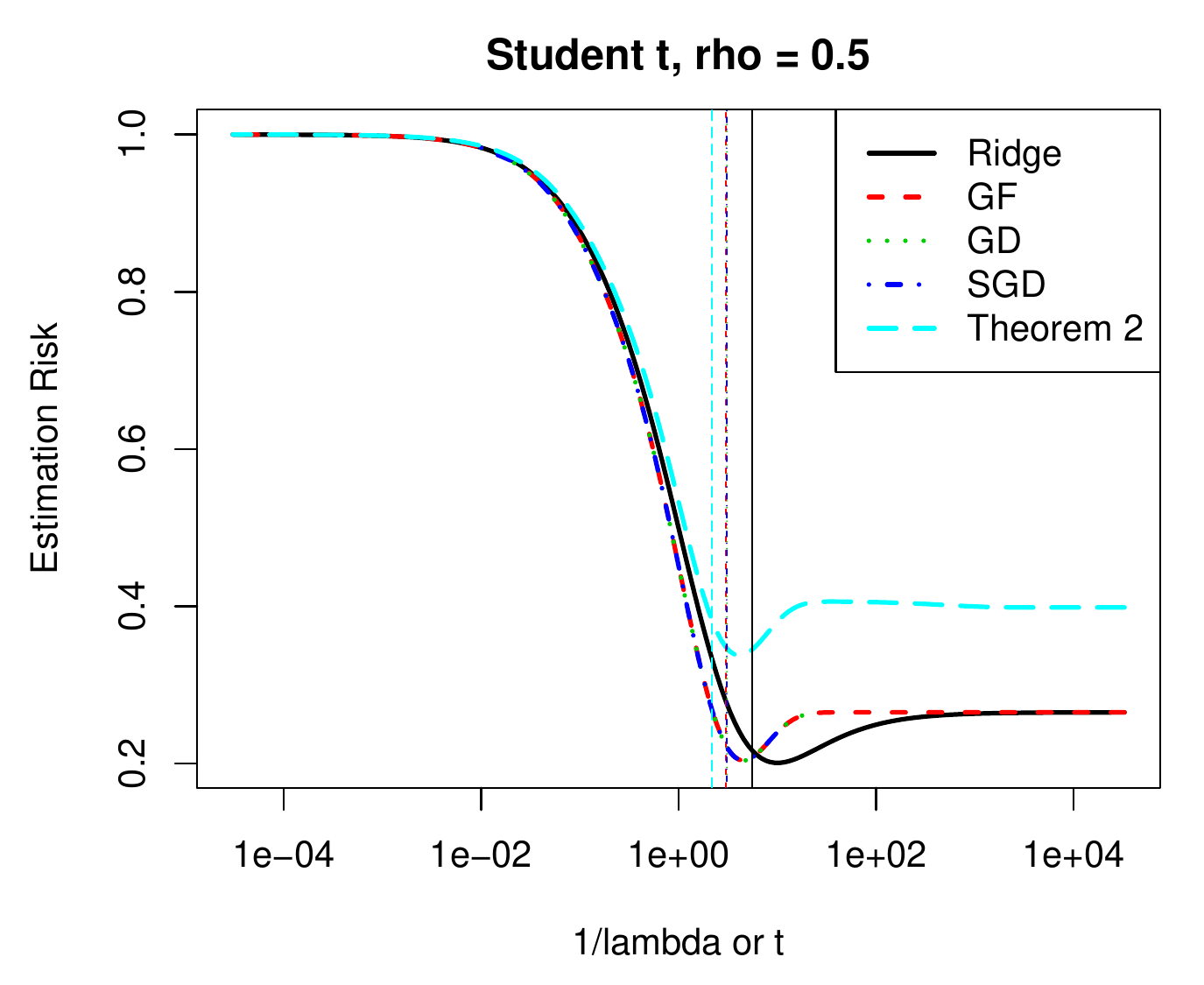}
\hfill
\includegraphics[width=0.33\textwidth]{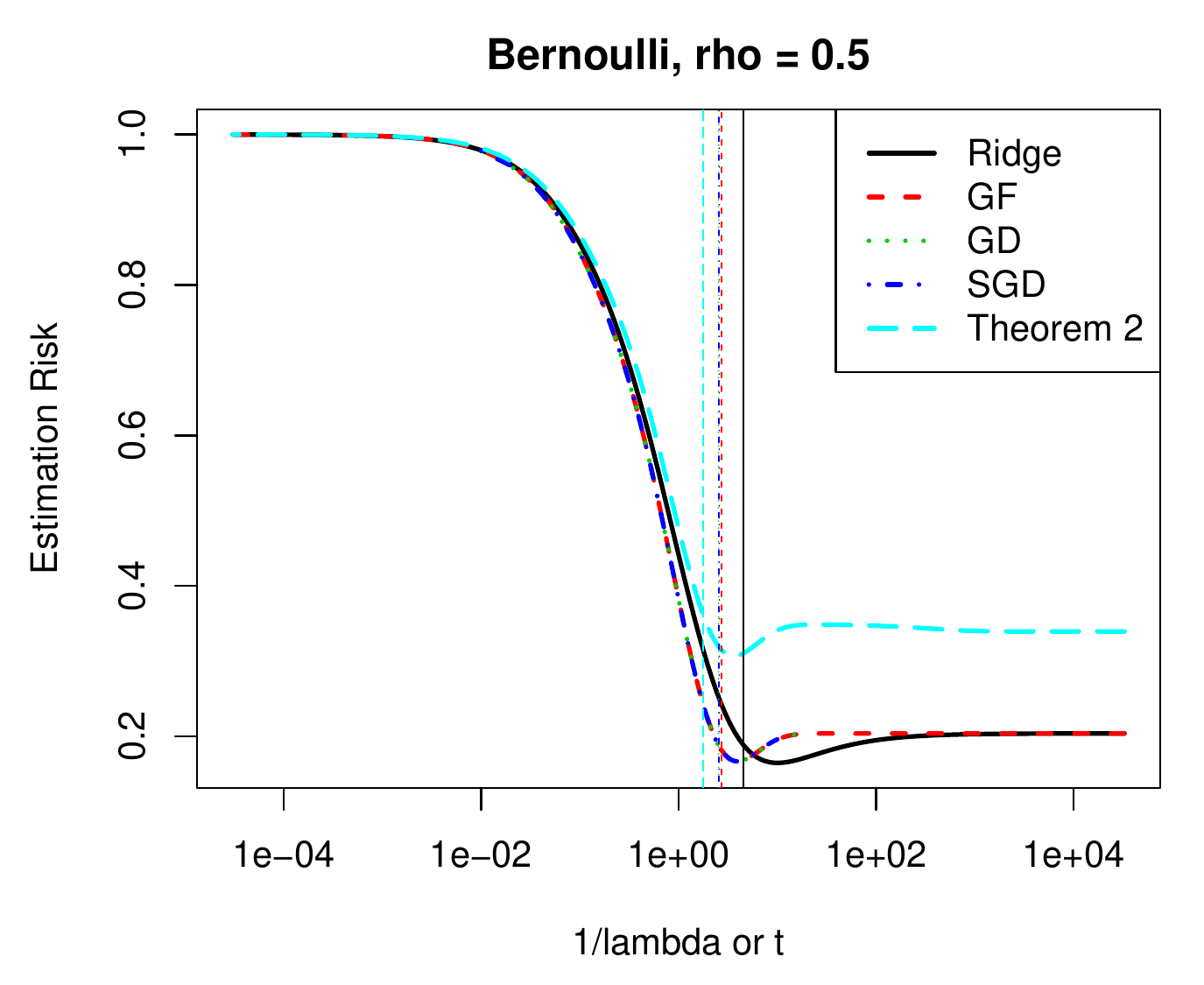}
}
%\vskip -0.4in
\caption{\textit{Risks for ridge regression, discrete-time SGD, stochastic gradient flow (as in Theorem \ref{thm:risk}), discrete-time gradient descent, and gradient flow on Gaussian, Student-t, and Bernoulli data.  The excess risk of stochastic gradient flow over ridge is given by the distance between the cyan and black curves.  The vertical lines show the stopping times that balance bias and variance.  In the first row, we set $n=500$, $p=100$, $m=20$, and $\epsilon$ following Lemma \ref{dyna:exp:loss}.  In the second, we set $n=100$, $p=10$, $m=10$, and $\epsilon$ following Lemma \ref{dyna:exp:loss}.}}
\label{fig:risk:others}
\end{center}
%\vskip -0.2in
\end{figure*}

%\clearpage
%\vfill
\begin{figure*}[h!]
%\vskip -0.2in
\begin{center}
\centerline{
\includegraphics[width=0.33\textwidth]{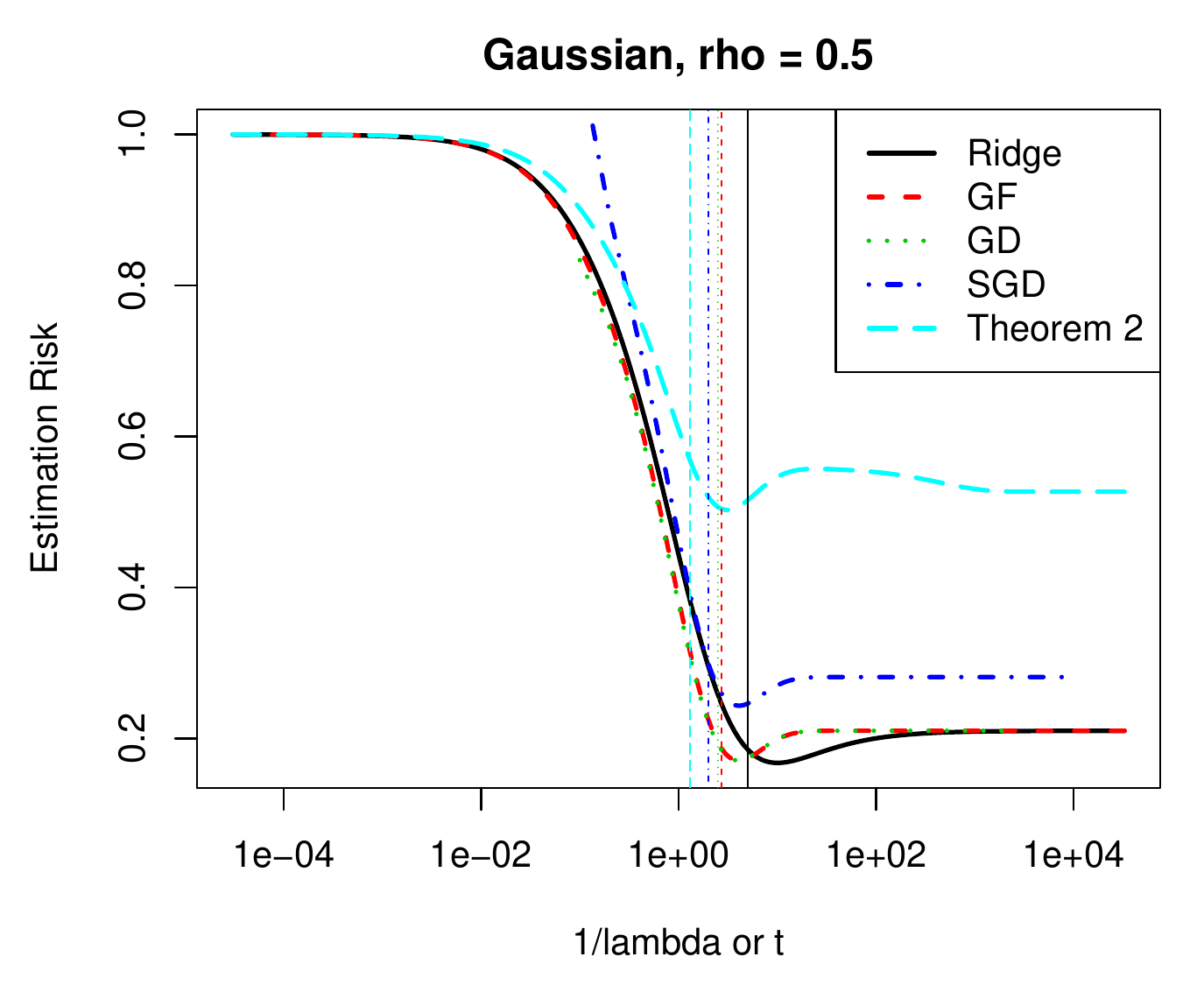} \hfill
\includegraphics[width=0.33\textwidth]{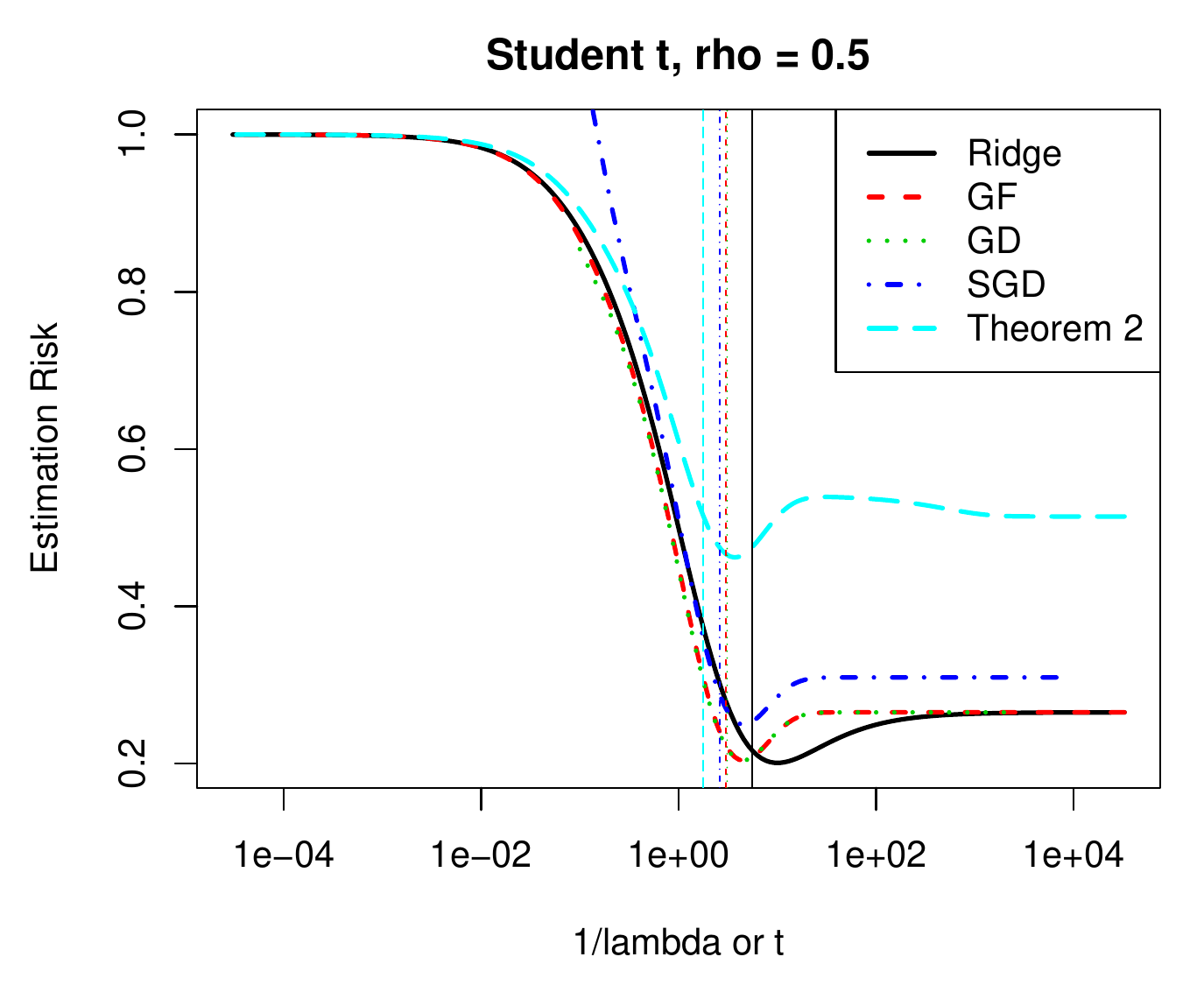} \hfill
\includegraphics[width=0.33\textwidth]{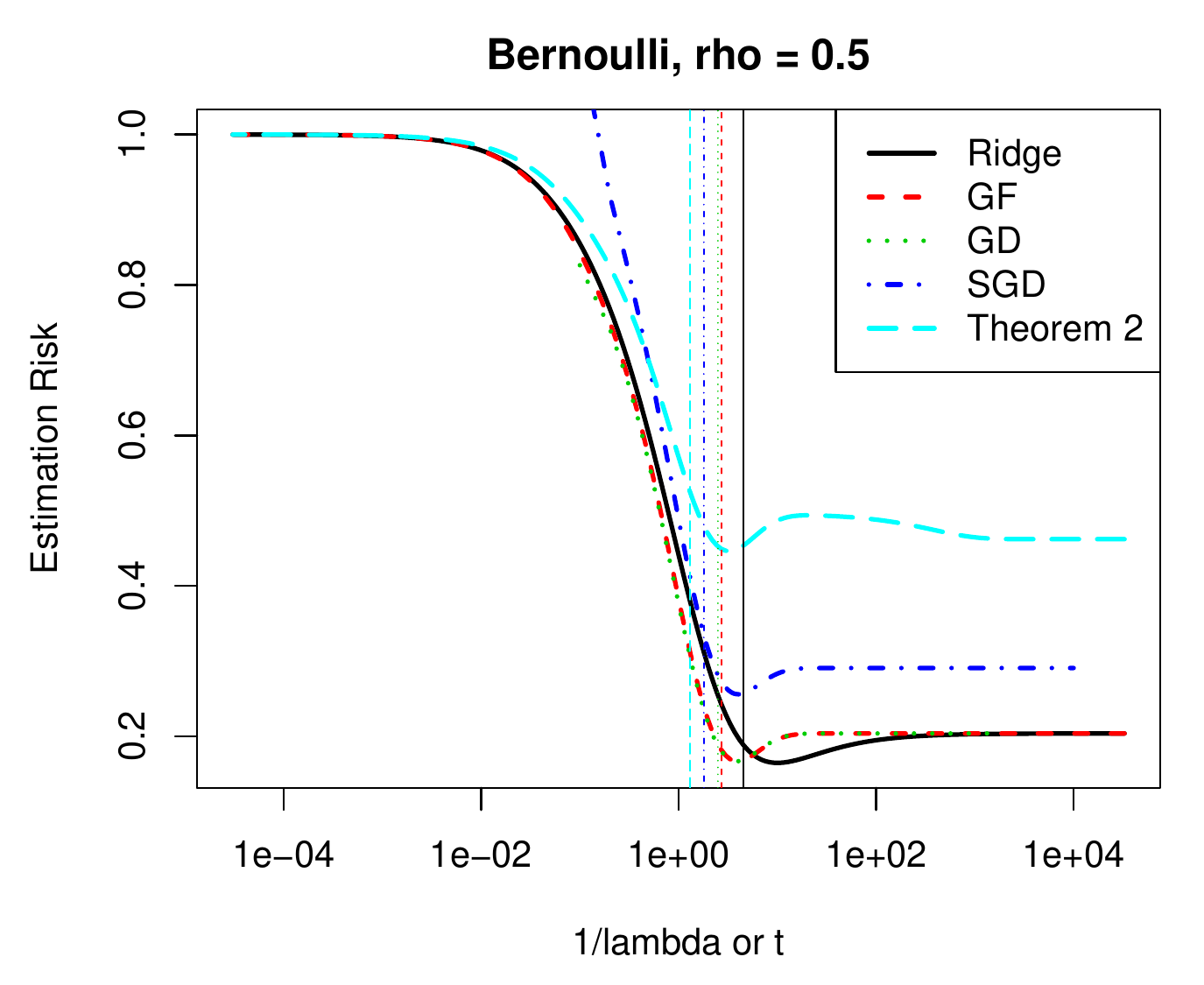}
}
%\vskip -0.4in
\caption{\textit{Risks for ridge regression, discrete-time SGD, stochastic gradient flow (as in Theorem \ref{thm:risk}), discrete-time gradient descent, and gradient flow on Gaussian, Student-t, and Bernoulli data, where $n=100$, $p=10$, $m=2$, and $\epsilon=0.1$.  The excess risk of stochastic gradient flow over ridge is given by the distance between the cyan and black curves.  The vertical lines show the stopping times that balance bias and variance.}}
\label{fig:risk:others:eps}
\end{center}
%\vskip -0.2in
\end{figure*}
%\vfill

\end{document}